\documentclass[dvipsnames]{article}

\usepackage{microtype}
\usepackage{graphicx}
\usepackage{subfigure}
\usepackage{booktabs} 

\usepackage{hyperref}
\usepackage{natbib}

\usepackage[accepted]{icml2024}

\usepackage{amsmath}
\usepackage{amssymb}
\usepackage{mathtools}
\usepackage{amsthm}

\usepackage[capitalize,noabbrev]{cleveref}

\theoremstyle{plain}
\newtheorem{theorem}{Theorem}[section]
\newtheorem{proposition}[theorem]{Proposition}
\newtheorem{lemma}[theorem]{Lemma}
\newtheorem{corollary}[theorem]{Corollary}

\theoremstyle{definition}
\newtheorem{definition}[theorem]{Definition}

\theoremstyle{remark}
\newtheorem*{remark}{Remark}

\usepackage[textsize=tiny]{todonotes}

\usepackage{multirow,threeparttable,makecell,colortbl,caption,float}
\usepackage{dsfont}
\usepackage{annotate-equations}
\usepackage[skins,listings]{tcolorbox}
\renewcommand{\eqnannotationfont}{\bfseries\small}
\usepackage{xcolor}
\usepackage[vlined,boxed,ruled,algo2e]{algorithm2e}

\newcommand{\red}[1]{\textcolor{BrickRed}{#1}}
\newcommand{\blue}[1]{\textcolor{NavyBlue}{#1}}
\newcommand{\purple}[1]{\textcolor{DarkOrchid}{#1}}

\newcommand{\update}[1]{\textcolor{black}{#1}}

\begin{document}

\twocolumn[
\icmltitle{How Graph Neural Networks Learn: Lessons from Training Dynamics}



\icmlsetsymbol{equal}{*}

\begin{icmlauthorlist}
\icmlauthor{Chenxiao Yang}{yyy,equal}
\icmlauthor{Qitian Wu}{yyy}
\icmlauthor{David Wipf}{comp}
\icmlauthor{Ruoyu Sun}{sch1,sch2}
\icmlauthor{Junchi Yan}{yyy}
\end{icmlauthorlist}

\icmlaffiliation{yyy}{School of Artificial Intelligence \&  Department of Computer Science and Engineering \& MoE Lab of AI, Shanghai Jiao Tong University}
\icmlaffiliation{comp}{Amazon Web Services}
\icmlaffiliation{sch1}{School of Data Science, The Chinese University of Hong Kong, Shenzhen}
\icmlaffiliation{sch2}{Shenzhen International Center for Industrial and Applied Mathematics, Shenzhen Research Institute of Big Data}

\icmlcorrespondingauthor{Junchi Yan}{yanjunchi@sjtu.edu.cn}

\icmlkeywords{Machine Learning, Graph Neural Networks, Optimization}

\vskip 0.3in]



\printAffiliationsAndNotice{$^*$ Work was partially done during an internship at Amazon Web Services.}  

\begin{abstract}
A long-standing goal in deep learning has been to characterize the learning behavior of black-box models in a more interpretable manner. For graph neural networks (GNNs), considerable advances have been made in formalizing what functions they can represent, but whether GNNs will learn desired functions during the optimization process remains less clear. To fill this gap, we study their training dynamics in function space. In particular, we find that the gradient descent optimization of GNNs implicitly leverages the graph structure to update the learned function, as can be quantified by a phenomenon which we call \emph{kernel-graph alignment}. We provide theoretical explanations for the emergence of this phenomenon in the overparameterized regime and empirically validate it on real-world GNNs. This finding offers new interpretable insights into when and why the learned GNN functions generalize, highlighting their limitations in heterophilic graphs. Practically, we propose a parameter-free algorithm that directly uses a sparse matrix (i.e. graph adjacency) to update the learned function. We demonstrate that this embarrassingly simple approach can be as effective as GNNs while being orders-of-magnitude faster.
\end{abstract}

\section{Introduction}
\emph{Graph Neural Networks (GNNs)}~\citep{gori2005new,scarselli2008graph,bruna2014spectral,kipf2016semi} represent network architectures for learning on entities with relations and interactions. In addition to their empirical success, the pursuit of theoretical understanding has also led researchers to dissect GNNs in terms of their representation powers (a.k.a. expressiveness)~\citep{maron2019provably,xu2018powerful,oono2019graph,chen2019equivalence}, which aim to answer what families of functions GNNs can represent or approximate. However, beyond these studies, it remains unclear whether GNNs will indeed learn the desired function during the training process. Filling this gap calls for more attention to the optimization of GNNs and effects of graph structures that implicitly bias the learning process towards certain solutions. In this paper, we analyze the training dynamics of GNNs in function space, aiming to answer

    \textit{Do GNNs indeed learn the desired function that can generalize during the training process, and if so, how?}

\begin{figure*}[t]
  \centering
  \subfigure[Training dynamics]{\hspace{-25pt}\includegraphics[width=0.69\linewidth]{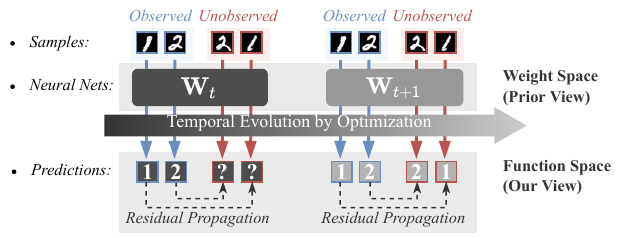}}
  \subfigure[Alignment of matrices]{\includegraphics[width=0.29\linewidth]{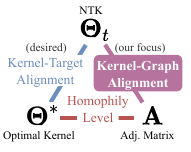}}
  \caption{\textbf{(a)} Training dynamics of GNNs in function space where residuals (i.e. difference between labels and predictions) propagate from observed to unobserved samples based on a kernel similarity measure.  \textbf{(b)} The kernel matrix $\mathbf \Theta_t$ naturally aligns with the adjacency matrix $\mathbf A$, which is favorable for generalization if $\mathbf A$ is inherently close to the optimal kernel $\mathbf \Theta^*$.}\label{fig_intro}
\end{figure*}

We note that the exact mathematical characterization of the optimization of neural networks, even for shallow ones, is prohibitive or even impossible due to their non-linearity. Existing work thus resorts to simplifications, such as removing activations (e.g., \citet{xu2021optimization} for GNNs) or training a single neuron (e.g., \citet{bai2019beyond,frei2020agnostic}). Moreover, even with these simplifications, there is generally no single straightforward answer to the above question; only insights from different angles exist for illuminating certain mechanisms in this complex process (e.g., \citet{cao2019towards,papyan2020prevalence,he2019local}). In this work, we provide initial answers to this open question by studying the role of graph structures in the training process. Particularly, we have the following informal statement:

\textit{The optimization of GNNs implicitly leverages the graph structure, as is quantified by kernel-graph alignment, to learn functions that can generalize.}

More formally, the so-called \emph{Kernel-Graph Alignment} characterizes such a graph implicit bias during training, where the so-called \emph{Neural Tangent Kernel (NTK)}~\citep{jacot2018neural} that controls the evolution of the learned GNN function tends to align with the message passing matrix (typically implemented as graph adjacency) used in GNNs' forward pass. This result is not only useful for understanding the optimization of GNNs but also allows us to explain a series of previously unresolved theoretical questions and develop practically useful algorithms in a principled manner. Our contributions are summarized as follows:

\textbf{Warm-Up Example: Residual Propagation.}\quad As a warm-up example, we propose a class of learning algorithms that replace the NTK matrix in the training dynamics with a sparse adjacency matrix, representing an extreme case where the kernel perfectly aligns with the graph. We dub this method \emph{Residual Propagation} since it simultaneously minimizes the loss and yields predictions by purely propagating residuals (a.k.a. errors) without the need for trainable parameters or back-propagation. The algorithm has interesting connections with two classic non-parametric algorithms: label propagation~\citep{zhou2003learning} and kernel regression~\citep{shawe2004kernel}. Strong empirical evidence across 15 benchmarks, including three OGBN datasets, demonstrates that this embarrassingly simple method has surprisingly good generalization performance, competing with non-linear GNNs. Additionally, compared with GNNs, it offers up to thousands of times speedup and requires ten times less memory due to its linear complexity. This algorithm serves as a simple example illustrating the inner mechanism of GNNs' optimization. (\textbf{Section~\ref{sec_mpperspective}})

\textbf{Implicit Leverage of Graphs in Optimization.}\quad Next, we formally study how the graph structure is leveraged in the optimization of GNNs. To make the analysis tractable, we examine the training dynamics of GNNs in the overparameterized regime~\citep{jacot2018neural,lee2019wide}, where the model asymptotically converges to its first-order Taylor expansion around its initialization as an approximation. Despite this approximation, the message passing modules and non-linearities in the model, which are critical for the success of GNNs, are preserved. A rich body of literature (e.g. see \citet{golikov2022neural} and references therein) also conduct analysis in this regime for gaining practical insights into otherwise prohibitive problems. Using this approach, we provide explicit mathematical formulas for recurrently computing the NTK of GNNs (which we call node-level GNTK) with arbitrary model depth and inputs. The formula demonstrates how the graph structure is naturally integrated into the kernel function, inducing the kernel-graph alignment phenomenon. Additional examples of shallow (two- and single-layer) GNNs with fixed inputs further show cases where the kernel function becomes equivalent to special forms of adjacency. (\textbf{Section~\ref{sec_gnndynamics}})

\textbf{Generalization and Failure Modes.}\quad Stepping further, we provide interpretable explanations of how the graph implicit bias in optimization leads to generalization of the learned GNN function, and why GNNs struggle with certain graph learning problems such as heterophily. As illustrated in Fig.~\ref{fig_intro}(b), we introduce another matrix called the optimal kernel matrix, which denotes if a pair of instances share the same label. The alignment of this matrix with the adjacency matrix quantifies the homophily level of the~\citep{zhu2020beyond}, a data characteristic that has been empirically shown to be highly relevant to GNNs' generalization performance. Intuitively, for homophilic graphs, a larger homophily level with good kernel-graph alignment indicates that the kernel function approaches the optimal one (the so-called kernel-target alignment~\citep{cristianini2001kernel}); this is desired for favorable generalization, as labels of training instances tend to flow to testing instances with the same label during optimization. Theoretically, we establish a strong correlation between generalization and homophily by deriving a data-dependent generalization bound that highly depends on the graph heterophily, and also showing that GNNs are the Bayesian optimal prior model architecture that minimizes the population risk on homophilic graphs. during optimization on datasets with diverse characteristics. (\textbf{Section~\ref{sec5}})

\textbf{Empirical Verification.}\quad To further verify the theory, we numerically study the evolution of real-world GNNs during the GD-based training process on synthetic and real-world datasets. We found that their NTKs indeed align with the message passing matrix used in the forward pass. On homophilic graphs, alignment with the graph promotes alignment with the optimal kernel matrix and thus improves generalization, whereas on heterophilic graphs, it adversely affects generalization. Additionally, we observed that GNNs are capable of gradually adapting themselves to align with the optimal kernel regardless of different graph structures, which might be of independent interest for understanding feature learning in GNNs. (\textbf{Section~\ref{sec_exp4theory}})


Finally, we note that our analysis may only explain a limited part of an under-explored complex problem, and hope our analytical framework to pave the way for dissecting other graph learning tasks (such as link prediction and representation learning), and to be used to gain insights into, and methodology to solve other practically relevant GNN issues. We conclude our paper by discussing such possibilities and more related works. Our codes are available at \url{https://github.com/chr26195/ResidualPropagation}. (\textbf{Section~\ref{sec_discussion}})

\section{Preliminary} \label{sec_preliminary}

\textbf{Notation and Setup.}\quad Given a training set with $n_l$ labeled instances $\mathbf X = \{\boldsymbol x_i\}_{i=1}^{n_l} \in \mathbb R^{n_l\times d}$ and $\mathbf Y = \{y_i\}_{i=1}^{n_l} \in \mathbb R^{n_l}$, we aim to learn a predictive function $f(\boldsymbol x)$ parameterized by weights $\mathbf W$. For (semi-)supervised learning, we minimize the squared loss $\mathcal L$ using \emph{Gradient Descent (GD)},
\begin{equation} \label{eqn_contiGD}
\mathcal L = \Vert \mathbf F_t - \mathbf Y \Vert^2/2, \quad {\partial \mathbf W_t}/{\partial t} = - \eta \nabla_{\mathbf W} {\mathcal L},
\end{equation}
where $\mathbf W_t$ and $\mathbf F_t = \{f_t(\boldsymbol x)\}_{\boldsymbol x \in \mathbf X} \in \mathbb R^{n_l}$ are weights and predictions for the training set at optimization time index $t$, and $\eta$ is the learning rate. Temporal discretization of this gradient flow system with time-step $\Delta t = 1$ yields the fixed step-size GD algorithm commonly used in practice. Let also $\mathbf X'$ and $\mathbf Y'$ denote testing instances, and $\bar{\mathbf X} = [\mathbf X, \mathbf X'] \in \mathbb R^{n\times d}$ (resp. $\bar{\mathbf Y}$) the concatenation of training and testing inputs (resp. labels), where $n$ is the full dataset size. For convenience, we generally refer to $f_t(\boldsymbol x)$ as prediction for a single data point, and allow it to depend also on other nodes' information. $\mathbf F_t$ and $\mathbf F'_t$ are predictions for the training and testing sets. This setup can be extended to other loss functions and multi-dimensional output (see discussions in Appendix.~\ref{app_loss}, \ref{app_multioutput}).

Similar to~\citep{xu2021optimization}, we focus on learning node representations, where instances (i.e. nodes) and their relations (i.e. edges) are described by an undirected graph $\mathcal G = (\mathcal V, \mathcal E)$, $|\mathcal V|=n$. The graph defines a symmetric adjacency matrix $\mathbf A \in \mathbb R^{n\times n}$ where $\mathbf A_{ij} = 1$ for a pair of connected nodes $(\boldsymbol x_i, \boldsymbol x_j)$ otherwise $0$. Based on the data split, we denote submatrices of $\mathbf A$ using $\mathbf A_{\mathbf X\mathbf X}$ and $\mathbf A_{\mathbf X'\mathbf X}$.
Our insights apply to both transductive (i.e. semi-supervised learning) and inductive settings (see Appendix.~\ref{app_induc_trans}), but will focus on the former case unless stated otherwise.

\textbf{Graph Neural Networks (GNNs)} are a class of network architectures for learning representations on graphs. For GNNs with ReLU activation, each layer can be written as
\begin{equation} \label{eqn_layer_gcn}
\mathbf Z^{(\ell)} = \operatorname{ReLU}(\mathbf A \mathbf Z^{(\ell-1)} \mathbf W^{(\ell)}) \in \mathbb R^{n\times m},
\end{equation}
where $\mathbf Z^{(\ell)}$ are node representations at the $\ell$-th layer with $\mathbf Z^{(0)} = \bar{\mathbf X}$, and $m$ is the model width. The definition of $\mathbf A$ could differ for different GNNs. Our analysis applies to arbitrary $\mathbf A$ and will not differentiate adjacency and the actual message passing matrix used in the implementation.
We denote GNN prediction for a single data point as $f(\boldsymbol x; \mathbf A)$.


\textbf{Label Propagation (LP)} represents a class of algorithms for semi-supervised learning, where labels $\mathbf Y$ propagate along edges to efficiently predict $\mathbf F'$. From~\citep{zhou2003learning}, the LP update equation can be written as $\operatorname{LP}(\mathbf Y; k, \alpha) =$
\begin{equation} \label{eqn_lp}
\begin{split}
 [\mathbf F_{k},\mathbf F'_{k}] &= \alpha {\mathbf A} ~[\mathbf F_{k-1},\mathbf F'_{k-1}] + (1-\alpha) [\mathbf Y, \mathbf 0],
\end{split}
\end{equation}
where $[\cdot,\cdot]$ is concatenation, {$k$ is the iteration number}, and $\alpha$ is a hyperparameter. As initialization $[\mathbf F_0, \mathbf F'_0] = [\mathbf Y, \mathbf 0]$, and after convergence $[\mathbf F_\infty, \mathbf F'_\infty] \propto (\mathbf I_n - \alpha \mathbf A)^{-1} [\mathbf Y, \mathbf 0]$. LP algorithms have found wide applicability due to their superior efficiency and scalability.

\section{Graph Implicit Bias in Training} \label{sec_mpperspective}

\update{We commence by providing a label propagation perspective on the evolution of a general parameterized model during GD-based optimization, whereby we propose a simple non-parametric algorithm for semi-supervised learning. The algorithm shows that explicitly leveraging graph structure in the training dynamics to update the learned function leads to satisfactory generalization performance that is comparable to non-linear GNNs. This serves as an illustrative example, later contributing to our understanding of the optimization of GNNs where they implicitly leverage graph structures.}

\subsection{Label Propagation View of Gradient Descent} \label{sec_rp}

On the training set, one can characterize the evolution of a general parameterized model (with no restriction on model architecture) $f(\cdot)$, induced by GD-based optimization that continuously updates the weights $\mathbf W_t$ as~\citep{jacot2018neural}:
\begin{equation} \label{eqn_nndynamics}
\begin{split}
\partial\mathbf F_t /\partial t &= \eta~  
\mathbf \Theta_t(\mathbf X, \mathbf X) \mathbf R_t\\
\mathbf \Theta_t(\mathbf X, \mathbf X) &\triangleq \nabla_{\mathbf W} \mathbf F_t^\top \nabla_{\mathbf W} \mathbf F_t \in \mathbb R^{n_l\times n_l},
\end{split}
\end{equation}
where $\mathbf R_t = \mathbf Y - \mathbf F_t \in \mathbb R^{n_l}$ denotes \emph{residuals}~\citep{hastie2009elements} (a.k.a. errors), the difference between ground-truth labels and model predictions. The so-called \emph{Neural Tangent Kernel (NTK)}~\citep{jacot2018neural} $\mathbf \Theta_t(\mathbf X, \mathbf X)$ is produced by the product of Jacobians, which is dependent on the network architecture and evolves over time due to its association with time-varying weights. Intuitively, it quantifies similarity between instances based on how differently their outputs change by an infinitesimal perturbation of weights. Specially, if the kernel is constant (such as inner-product kernel for linear models), (\ref{eqn_nndynamics}) reduces to the training dynamics of kernel regression~\citep{shawe2004kernel}. The derivation of (\ref{eqn_nndynamics}) mainly relies on the chain rule; we reproduce it in Appendix~\ref{proof_rp_ntk} for self-containment.

\textbf{Residual Dynamics.}\quad While (\ref{eqn_nndynamics}) has found usage for analyzing the convergence of empirical risk~\citep{du2019gradientb,arora2019fine} and the spectral bias of deep learning~\citep{mei2019mean,cao2019towards}, it is restricted to a limited set of samples (i.e. training set). To see how the model evolves on arbitrary inputs for fully characterizing the learned function, we extend (\ref{eqn_nndynamics}) to accommodate unseen samples (which could be chosen arbitrarily).
Specifically, let $\mathbf R'_t = \mathbf Y' - \mathbf F'_t$ denote residuals for the testing set, and after temporal discretization, the ODE in (\ref{eqn_nndynamics}) can be rewritten neatly using a single variable residual $\mathbf R$ (see derivation in Appendix~\ref{proof_rp_ntk}),
\begin{equation} \label{eqn_inference}
\left[\mathbf R_{t+1}, {\mathbf R'_{t+1}}\right] = -\eta~ \mathbf \Theta_t(\bar{\mathbf X}, \bar{\mathbf X}) \left[{\mathbf R_t}, \mathbf 0\right] + \left[\mathbf R_t, \mathbf R'_{t}\right],
\end{equation}
where $\mathbf \Theta_t(\bar{\mathbf X}, \bar{\mathbf X}) \triangleq \nabla_{\mathbf W} [\mathbf F_t, \mathbf F'_t]^\top \nabla_{\mathbf W} [\mathbf F_t, \mathbf F'_t] \in \mathbb R^{n\times n}$ is the NTK matrix between training and testing sets. The $n\times n$ matrix will be henceforth abbreviated as $\mathbf \Theta_t$. The equation can be viewed as propagating residuals unidirectionally from training to arbitrary unseen samples based on a similarity measure, controlling the evolution of the learned function.

To provide more intuitions into the inner mechanism of (\ref{eqn_inference}), we rewrite it for an arbitrary unseen data point $\boldsymbol x'$:

\begin{equation} \label{eqn_single_rp} \nonumber
\eqnmarkbox{f}{f_{t+1}(\boldsymbol x')} = 
f_{t}(\boldsymbol x') + \eta \sum_{\boldsymbol x \in \mathbf X} \mathbf \Theta_t(\boldsymbol x, \boldsymbol x')(
\eqnmarkbox{y}{y(\boldsymbol x)} - f_t(\boldsymbol x) 
),
\end{equation}  {\tikzset{annotate equations/arrow/.style={->}} \annotatetwo[yshift=0.4em]{above}{y}{f}{(Ground-truth) label propagation in optimization}}
\hspace{-8pt}where $y(\boldsymbol x)$ is the ground-truth label for $\boldsymbol x$, and $\mathbf \Theta_t(\boldsymbol x, \boldsymbol x') = \nabla_{\mathbf W} f(\boldsymbol{x})^\top \nabla_{\mathbf W} f(\boldsymbol{x}')$. Intuitively, for an unseen instance $\boldsymbol x'$ that is more `similar' to $\boldsymbol x$, more ground-truth label information $y(\boldsymbol x)$ will then propagate to $\boldsymbol{x}'$, and vice versa (illustrated in Fig.~\ref{fig_intro}(a)). For $f_t(\boldsymbol x)\neq 0$, the ground-truth label is adjusted by subtracting current model prediction, i.e. $y(\boldsymbol x) - f_t(\boldsymbol x)$, enabling the propagation process to diminish progressively as errors or residuals are minimized.

\update{
Drawing upon an analogy between (\ref{eqn_inference}) induced by optimization, and instance-wise propagation schemes commonly seen in graph learning, we have the following hypothesis: \textit{For GNNs defined in Section \ref{sec_preliminary}, its optimization process implicitly leverages the graph structure to update the learned function by aligning their $\mathbf \Theta_t$ with certain forms of $\mathbf A$, which is also a key factor contributing to their good generalization performance}.}

\subsection{Residual Propagation} \label{sec_rp_theorem}
To test this hypothesis, we propose a toy algorithm that explicitly introduces graph structure information into training dynamics in (\ref{eqn_inference}), by replacing the NTK matrix $\mathbf \Theta_t$ with high-order graph adjacency matrix $\mathbf A^K$. While original (\ref{eqn_inference}) is expensive to run, such replacement allows us to actually implement it as a practically useful semi-supervised algorithm that can efficiently run, by taking advantage of the fact that $\mathbf A$ is sparse. We dub this algorithm as \emph{Residual Propagation (RP)}, whose update equation is
\begin{equation} \label{eqn_rp}
\left[\mathbf R_{t+1}, \mathbf R'_{t+1}\right] =  -\eta \mathbf A^{K}  [\mathbf R_t, \mathbf 0]  + \left[\mathbf R_t, \mathbf R'_{t}\right].
\end{equation}
At initialization, $\mathbf F_0$ and $\mathbf F'_0$ are defined as $\mathbf 0$, and unknown testing labels are defined as $\mathbf Y' = \mathbf 0$. One can conveniently convert $\mathbf R'_t$ back to predictions at a certain time step. More generally, one can flexibly replace the term $\mathbf A^{K}$ with other matrices indicating similarity of samples to broaden the use cases of the algorithm.

Intriguingly, we show the RP algorithm has interesting connections with various classic methods including LP~\citep{zhou2003learning} and kernel regression~\citep{shawe2004kernel}, though they emerge from very different contexts.

\begin{proposition}[Connection with Label Propagation] \label{thm_connectlp}
The first step of RP in (\ref{eqn_rp}) yields identical classification results as LP in (\ref{eqn_lp}) (with $\alpha = 1$ and $k = K$):
\begin{equation}
\begin{split}
\text{(First Step of RP):~~~}&[\mathbf F_{1}, \mathbf F'_1] = \eta \mathbf A^K  [\mathbf Y, \mathbf 0],\\
\text{(Label Propagation):~~~}&\operatorname{LP}(\mathbf Y;K, 1) = \mathbf A^K [\mathbf Y, \mathbf 0],
\end{split}
\end{equation}
and each of subsequent step of RP can also be viewed as LP on adjusted ground-truth labels, i.e. $\mathbf Y - \mathbf F_t = \mathbf R_t$. 
\end{proposition}

Besides the first step, each of subsequent step of RP can also be viewed as LP on adjusted ground-truth labels, i.e. $\mathbf Y - \mathbf F_t = \mathbf R_t$. This result shows RP encompasses LP as a special case; such a connection further motivates a generalized version of RP by combining with other off-the-shelf LP variants (see references in Appendix~\ref{app_rw_lp}):
\begin{equation} \label{eqn_generalrp}
\left[\mathbf R_{t+1}, \mathbf R'_{t+1}\right] =-\eta \operatorname{LP}^*(\mathbf R_t)  + \left[\mathbf R_t, \mathbf R'_{t}\right].
\end{equation}
where $\operatorname{LP}^*(\cdot): \mathbb R^{n_l} \rightarrow \mathbb R^{n}$ is a general LP function that takes as input ground-truth labels and outputs predictions.

\begin{table*}[t]
\centering
\caption{Empirical evaluation of RP on \texttt{OGB} datasets. Accuracy is reported for \texttt{Arxiv} and \texttt{Products}, and ROC-AUC for \texttt{Proteins}. Last three rows compare RP against full-batch GNN. Results of some baseline are from the official leaderboard.}  \label{tbl_bench}
\resizebox{0.95\textwidth}{!}{
\begin{tabular}{@{}c|c|cc|cc|cc|c@{}}
\toprule
\textbf{Model} & \textbf{Feat.} & \multicolumn{2}{c|}{\texttt{Arxiv}} & \multicolumn{2}{c|}{\texttt{Proteins}} &  \multicolumn{2}{c|}{\texttt{Products}}&\textbf{\# Param.}\\
& & Validation & Test & Validation & Test & Validation & Test & \\
\midrule
MLP & $\mathbf X$  &57.65 $\pm$ 0.12 & 55.50 $\pm$ 0.23 & 77.06 $\pm$ 0.14& 72.04 $\pm$ 0.48 &75.54 $\pm$ 0.14& 61.06 $\pm$ 0.08  & $O(\ell m^2)$ \\
LinearGNN & $\mathbf X, \mathbf A$   & 70.67 $\pm$ 0.02& 69.39 $\pm$ 0.11 & 66.11 $\pm$ 0.87 & 62.89 $\pm$ 0.11 & 88.97 $\pm$ 0.01 & 74.21 $\pm$ 0.04  &$O(dc)$\\
GNN & $\mathbf X, \mathbf A$   & \cellcolor{Goldenrod!40} {73.00 $\pm$ 0.17} & \cellcolor{Goldenrod!40}{71.74 $\pm$ 0.29} & 79.21 $\pm$ 0.18& 72.51 $\pm$ 0.35 & \cellcolor{Goldenrod!40}{92.00 $\pm$ 0.03}& \cellcolor{gray!30}75.64 $\pm$ 0.21  & $O(\ell m^2)$\\
LP & $\mathbf A$ & 70.14 $\pm$ 0.00 & 68.32 $\pm$ 0.00 & \cellcolor{gray!30}83.02 $\pm$ 0.00& \cellcolor{gray!30}74.73 $\pm$ 0.00 & 90.91 $\pm$ 0.00& 74.34 $\pm$ 0.00  & 0\\
\midrule
RP (ours) & $\mathbf A$  &\cellcolor{gray!30}71.37 $\pm$ 0.00& \cellcolor{gray!30}70.06 $\pm$ 0.00 & \cellcolor{Goldenrod!40}{85.19 $\pm$ 0.00}& \cellcolor{Goldenrod!40}{78.17 $\pm$ 0.00} &\cellcolor{gray!30} 91.31 $\pm$ 0.00& \cellcolor{Goldenrod!40}{78.25 $\pm$ 0.00} & 0\\
Speedup / step & & \multicolumn{2}{c|}{$\times$ 14.48} & \multicolumn{2}{c|}{$\times$ 14.00} & \multicolumn{2}{c|}{$\times$ 12.46} & \\
Time to Acc. & & \multicolumn{2}{c|}{$\times$ 0.01461} & \multicolumn{2}{c|}{$\times$ 0.00008} & \multicolumn{2}{c|}{$\times$ 0.00427} &\\
Memory & & \multicolumn{2}{c|}{$\times$ 0.094} & \multicolumn{2}{c|}{$\times$ 0.363} & \multicolumn{2}{c|}{$\times$ 0.151} &\\
\bottomrule
\end{tabular}}
\end{table*}

\begin{figure*}[h]
\centering
\begin{minipage}[t]{\linewidth}
\begin{minipage}[t]{0.33\linewidth}
\centering
\includegraphics[width=0.95\textwidth]{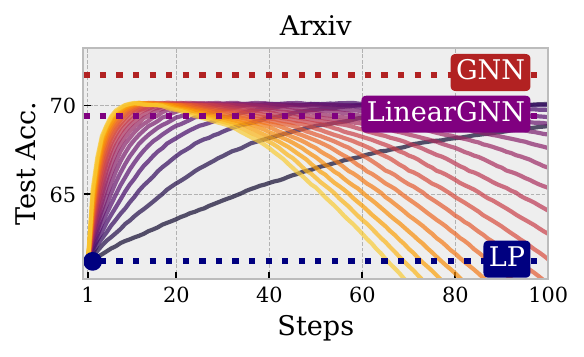}
\end{minipage}%
\begin{minipage}[t]{0.33\linewidth}
\centering
\includegraphics[width=0.95\textwidth]{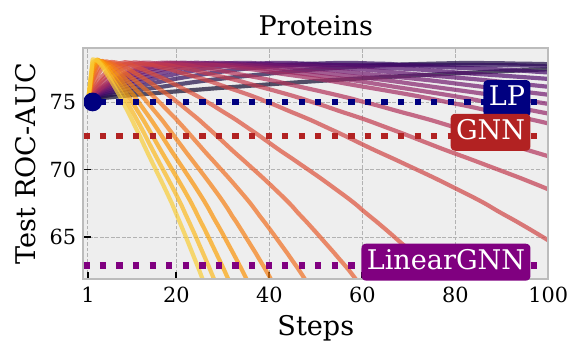}
\end{minipage}%
\begin{minipage}[t]{0.33\linewidth}
\centering
\includegraphics[width=0.95\textwidth]{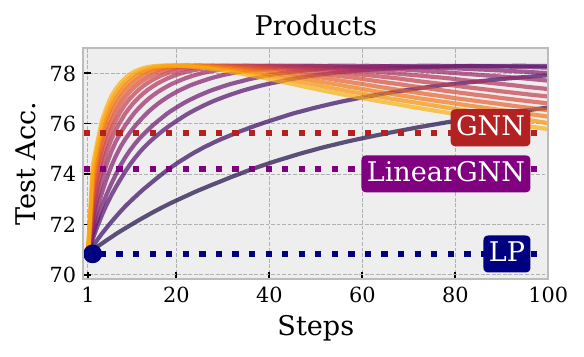}
\end{minipage}%
\end{minipage}
\caption{Learning curves of RP and comparison with the performance of LP ($\alpha=1$), linear GNN and deep GNN. Transition from \textcolor[RGB]{200, 171, 37}{yellow} to \textcolor[RGB]{102, 60, 130}{purple} denotes RP with decreasing step size $\eta$.}
\label{fig_rpcurve}
\end{figure*}

\begin{theorem}[Convergence \& Connection with Kernel Regression]\label{thm_converge}
For RP in (\ref{eqn_rp}) and sufficiently small step size $\eta < 2 \sigma^{-1}_{max}[\mathbf A^K_{\mathbf X\mathbf X}]$, {where $\mathbf A^K_{\mathbf X\mathbf X}$ is a submatrix of $\mathbf A^K$, and $\sigma_{max}$ is its largest eigenvalue}, $\mathbf R_t$ and $\mathbf R'_t$ converge for positive definite $\mathbf A^K_{\mathbf X\mathbf X}$ or positive semi-definite $\mathbf A^K$. Upon convergence in the former case, the predictions are equivalent to the kernel regression solution w.r.t. kernel $\kappa(\boldsymbol x, \boldsymbol x') \triangleq \mathbf A^K_{\boldsymbol x\boldsymbol x'}$
\begin{equation}
\mathbf F_\infty = \mathbf Y, \quad \mathbf F'_\infty = \mathbf A^K_{\mathbf X'\mathbf X}(\mathbf A^K_{\mathbf X\mathbf X})^{-1} \mathbf Y.
\end{equation}
\end{theorem}
See proof and a more comprehensive discussion of the convergence of RP and its connection with classic algorithms in Appendix~\ref{app_converge}. Different from kernel regression, RP might not necessarily converge since it does not restrict the propagation matrix to be PSD or symmetric. But intriguingly, it can still achieve satisfactory generalization performance by stopping at the step with peak validation performance.

\subsection{Preliminary Empirical Verification}

Since RP represents an extreme case where the kernel function perfectly aligns with the graph, we test its performance against GNNs to empirically verify whether the graph implicit bias in optimization alone is crucial for satisfactory performance.

\textbf{Setup.}\quad We compare RP with some standard GNN architectures (LinearGNN~\citet{wu2019simplifying} and GCN~\citet{kipf2016semi}) on a diverse set of 15 datasets, including three challenging \texttt{OGB}~\citep{hu2020open} datasets \texttt{Arxiv}, \texttt{Proteins}, \texttt{Products} with up to millions of nodes and edges. Due to space limit, we report results on OGB datasets in the main text, and defer the rest and experimental details to Appendix~\ref{app_addexp} and \ref{app_detail1}.

\textbf{Results.}\quad In Table~\ref{tbl_bench}, the proposed RP demonstrates competitive performance. 
As depicted in Fig.~\ref{fig_rpcurve}, RP achieves the same performance as LP using one step, and quickly increases until reaching its peak performance, which often surpasses GNNs. Specifically, on \texttt{Proteins} where the graph contains relatively richer structural information, a single step of RP exceeds a well-trained deep GNN, while in \texttt{Products}, merely four steps of RP exceeds the GNN. Moreover, RP does not require trainable parameters and boosts speed, with each step being more than 10 times faster than each gradient descent step for training GNN. Furthermore, RP can achieve its peak performance using less steps, and thus overall takes significantly less time (up to $1000$ times less) to attain GNN-level performance. RP also inherits the scalability of LP and only requires storage space for predictions; thus it consumes up to $10$ times less memory in practice.

In Appendices~\ref{app_addexp1} and \ref{app_addexp2} respectively, we discuss generalized RP that combines with kernels (e.g. Gaussian kernel) to incorporate node features. The corresponding update equation can be re-written as $\left[\mathbf R_{t+1}, \mathbf R'_{t+1}\right] = $
\begin{equation}
-\eta \mathbf A^{K} \mathbf K(\bar{\mathbf X}, \bar{\mathbf X})\mathbf A^{K}  [\mathbf R_t, \mathbf 0]  + \left[\mathbf R_t, \mathbf R'_{t}\right],
\end{equation}
where $\mathbf K(\bar{\mathbf X}, \bar{\mathbf X})$ could be specified as arbitrary kernel functions (such as Sigmoid, Gaussian, etc.) Additionally, we test $12$ datasets including homophilic (\texttt{Cora}, \texttt{Citeseer}, \texttt{Pubmed}, \texttt{Computer}, \texttt{Photo}, \texttt{CS}, \texttt{Physics}) and heterophilic (\texttt{roman-empire}, \texttt{amazon-ratings}, \texttt{minesweeper}, \texttt{tolokers}, \texttt{questions}) ones where we find RP can still outperform most popular GNNs.

\section{Analysis in Overparameterized Regime} \label{sec_gnndynamics}

We next theoretically verify that the optimization of GNNs indeed implicitly leverage graph structure for updating the learned function. Further empirical verification will be deferred to the next section.

\subsection{Insights from Explicit Formula of GNNs' NTK}

Similar to the training dynamics of general parameterized models in (\ref{eqn_inference}), the training dynamics of GNNs in node-level tasks is characterized by their NTK defined as follows:

\begin{definition}[Node-Level GNTK]
For a $\ell$-layer GNN in node-level tasks defined in Sec.~\ref{sec_preliminary}, the NTK is defined as
\begin{equation} \label{eqn_nodelevelgntk}
\mathbf \Theta^{(\ell)}_t(\boldsymbol x, \boldsymbol x'; \mathbf A) = \nabla_{\mathbf W} f(\boldsymbol x; \mathbf A)^\top \nabla_{\mathbf W} f(\boldsymbol x'; \mathbf A),
\end{equation}
which we refer to as Node-Level Graph Neural Tangent Kernel (GNTK), or simply NTK of GNNs, to differentiate it with the graph-level GNTK for graph-level tasks initially proposed in \citet{du2019graph}.
\end{definition}

How GNNs evolve during training also follows (\ref{eqn_inference}). Ideally, one might want to show the precise mathematical characterization of (\ref{eqn_nodelevelgntk}) and its connection with $\mathbf A$, which however is impossible in principle due the complexity of optimizing non-linear neural networks. \update{To make the analysis of non-linear models tractable, a previous work~\citep{xu2021optimization} removes all activations such that the GNN boils down to a linear model. In comparison, we adopt an assumption that makes the analysis tractable in a way that still preserves model's non-linearity.} Specifically, we study node-level GNTK in overparameterized regimes, where the model width $m$ tends to infinity, and consequently, neural networks asymptotically converge to its first order Taylor expansion around its initialization (e.g. \citet{jacot2018neural,lee2019wide}), i.e. a kernel regression predictor. In this case, the node-level GNTK is a constant kernel $\mathbf \Theta^{(\ell)}$ (without subscript $t$). \update{We next represent the explicit mathematical formula, which applies to arbitrary layer number $\ell$, input features $\mathbf X$, graph structure $\mathcal G$, and various GNN architectures that conform to the definition in Section~\ref{sec_preliminary}.}

Similar to fully-connected neural networks, the NTK of a $\ell$-layer GNN can be recursively computed based on the NTK of a $(\ell-1)$-layer GNNs. The layer-wise formula can be decomposed into two steps that respectively correspond to \emph{Transformation} (i.e. $\mathbf Z\leftarrow \operatorname{ReLU}(\mathbf Z\mathbf W)$) and \emph{Propagation} (i.e. $\mathbf Z\leftarrow \mathbf A \mathbf Z$) in the GNN architecture. As additional elements in the computation, we denote a node-level GNTK for GNN without propagation at $\ell$-th layer as {$\small \bar{\mathbf \Theta}^{(\ell)}$}, the covariance matrix of the $\ell$-th layer's outputs with (and w/o) propagation as {${\mathbf \Sigma}^{(\ell)}$} (and {$\bar{\mathbf \Sigma}^{(\ell)}$}), and the covariance matrix of the derivative to the $\ell$-th layer as {$\dot{{\mathbf \Sigma}}^{(\ell)}$}. Then the propagation and transformation steps in each layer respectively correspond to (we concisely show the key steps here and defer the complete formula to Appendix~\ref{app_gntk}):
\begin{equation}
\mbox{(\textit{Transformation})} \quad \bar{\mathbf \Theta}^{(\ell)} = {\mathbf \Theta}^{(\ell-1)} \odot \dot{\mathbf{\Sigma}}^{(\ell)} + \bar{\mathbf{\Sigma}}^{(\ell)} 
\end{equation}
\begin{equation} 
{\mbox{(\textit{Propagation})}} \quad\left\{\begin{array}{l}
    {\mathbf{\Sigma}}^{(\ell)}= {\mathbf A} ~\bar{\mathbf{\Sigma}}^{(\ell)} {\mathbf A} \\
    {\mathbf \Theta}^{(\ell)}={\mathbf A} ~\bar{\mathbf \Theta}^{(\ell)} {\mathbf A}.
    \end{array}\right.
\end{equation}
\paragraph{Implications.} Compared with the computation for NTK of a fully-connected neural network~\citep{jacot2018neural}, the node-level GNTK has an equivalent transformation step, while its uniqueness stems from the propagation step, whereby the adjacency matrix $\mathbf A$ (or propagation matrix more generally) naturally integrates into the kernel similarity measure. Consequently, this kernel function inherently accommodates a graph implicit bias, and thus the gradient descent optimization of GNNs also tends to follow the trajectory regulated by the graph, similar to the behavior of the RP algorithm. To give more concrete examples, we provide case studies and show how the training dynamics of certain shallow GNNs, given fixed inputs (that are practically used), can be exactly described by the framework of generalized RP in (\ref{eqn_generalrp}). 

\update{We will further provide analysis for real-world finite width GNNs' NTK in Section~\ref{sec5}, to empirically corroborate our insights. We believe similar results can be potentially obtained using more advanced techniques, e.g.~\citep{bai2019beyond}, which we leave as future works.}

\subsection{Illustrative Examples} 

Given that our primary focus centers on the role of graphs in GNNs (as without them, GNNs are largely equivalent to MLPs), we exclude external node features and instead define inputs as either: 1) an identity matrix {$\bar{\mathbf X} \triangleq \mathbf I_{n}$} that assigns each node a one-hot vector as indication of its unique identity~(as sometimes assumed in practice \citep{kipf2016semi,zhu2021graph}), which can be viewed as learning a unique embedding for each node by treating the first-layer weights as an embedding table; 2) fixed node embeddings from graph spectral decomposition {$\bar{\mathbf X} \triangleq \mathop{\arg\min}_{\mathbf B} \|\mathbf A - \mathbf B \mathbf B^\top\|_F^2$}, which aligns with various network embedding approaches based on definitions of $\mathbf A$~\citep{qiu2018network}.

\paragraph{Two-Layer GNN.} 
Following the setup from prior work on fully-connected neural networks~\citep{arora2019fine}, we consider the training of two-layer GNNs where first-layer weights are optimized:

\begin{theorem}[Two-Layer GNN] \label{thm_twolayergnn}
For an infinite width two-layer GNN defined as $[\mathbf F, \mathbf F'] = \mathbf A \operatorname{ReLU}(\mathbf A \bar{\mathbf X} \mathbf W^{(1)}) \mathbf W^{(2)} / \sqrt{m}$ with $\bar{\mathbf X} = \mathbf I_{n}$ as inputs and standard NTK parameterization, its training dynamics by optimizing $\mathbf W^{(1)}$ can be written as a generalized RP process
\begin{eqnarray}
&\left[\mathbf R_{t+1}, \mathbf R'_{t+1}\right] = -\eta \mathbf A (\mathbf A^2 \odot {\mathbf S})\mathbf A [\mathbf R_t, \mathbf 0]  + \left[\mathbf R_t, \mathbf R'_{t}\right],\nonumber\\
&{\mathbf S}_{ij} = \left({\pi - \operatorname{arccos}(\frac{\mathbf A_i^\top \mathbf A_j}{\|\mathbf A_i\| \|\mathbf A_j\|})}\right)/{2\pi}.
\end{eqnarray}
The matrix $\mathbf S$ reweights each element in $\mathbf A^2$ by the similarity of neighborhood distributions of two nodes. For $\bar{\mathbf X} = \mathop{\arg\min}_{\mathbf B} \|\mathbf A - \mathbf B \mathbf B^\top\|_F^2$, the propagation matrix is replaced by $\mathbf A(\mathbf A^3\odot \tilde{\mathbf S}) \mathbf A$ where $\tilde{\mathbf S}$ is another reweighting matrix (details and proof in Appendix~\ref{sec_twolayergnn}).
\end{theorem}

\paragraph{Deep and Wide GNNs.}
Pushing further, we can also characterize the evolution of arbitrarily deep GNNs where feature propagation is applied at the last layer (e.g. ~\citet{klicpera2018predict,liu2020towards,spinelli2020adaptive,chien2020adaptive}):
, i.e. $f(\mathbf X;\mathbf A, \mathbf W) = \mathbf A^{\ell} \operatorname{MLP}(\mathbf X; \mathbf W)$.

\begin{theorem}[Deep and Wide GNN Dynamics] \label{thm_decoupledgnn}
For arbitrarily deep and infinitely-wide GNNs with feature propagation deferred to the last layer, i.e. $[\mathbf F, \mathbf F'] = \mathbf A^{\ell} \operatorname{MLP}(\bar{\mathbf X})$ with $\bar{\mathbf X} = \mathbf I_{n}$, the training dynamics that result from optimizing MLP weights can be written as the generalized RP process $\left[\mathbf R_{t+1}, \mathbf R'_{t+1}\right] = $
\begin{equation}
-\eta \mathbf A^{\ell} (\mathbf I_n + c\mathbf 1 \mathbf 1^\top) \mathbf A^{\ell}  [\mathbf R_t, \mathbf 0]  ~+~ \left[\mathbf R_t, \mathbf R'_{t}\right],
\end{equation}
where $c \geq 0$ is a constant determined by the model depth, and $\mathbf 1$ is an all-$1$ column vector.
\end{theorem}

\paragraph{Linear GNNs.}
Another simple and interesting example is linear GNN (specifically SGC~\citet{wu2019simplifying}), whose training dynamics is equivalent to a special case of RP in (\ref{eqn_rp}):


\begin{corollary}[One-Layer GNN] \label{thm_lineargnn}
The training dynamics of the linear GNN $[\mathbf F, \mathbf F'] = \mathbf A^\ell \bar{\mathbf X}\mathbf W$ is identical to the basic version of RP in (\ref{eqn_rp}) with $K = 2\ell$ for input features $\bar{\mathbf X} = \mathbf I_{n}$, and $K = 2\ell + 1$ for $\bar{\mathbf X} = \mathop{\arg\min}_{\mathbf B} \|\mathbf A - \mathbf B \mathbf B^\top\|_F^2$ and positive semi-definite $\mathbf A$. (Proof in Appendix~\ref{app_lineargnn})
\end{corollary}

\begin{remark}
Despite this equivalence, it is important to note that this specific linear GNN (on our given input features) is not as lightweight as it may appear, since its parameter number scales with the size of dataset (even reaching orders of magnitude larger than deep GNN models), and the full-rank spectral decomposition of $\mathbf A$ is computationally very expensive (i.e. $n^3$) for large graphs. In contrast, RP efficiently yields identical results to this heavily parameterized GNN without actually training parameters or decomposing matrix. In practice, RP also consistently performs better than linear GNN, since real-world input feature is commonly a tall matrix (i.e. $n>d$), which makes linear models sub-optimal (see details in Appendix). 
\end{remark}

\update{Note that there might exist many other examples where the NTK of GNNs is identical to special forms of adjacency. An exhaustive list of them is prohibitive. We refer readers to Appendix~\ref{app_decoupled} for analysis of another type of deep GNN where feature propagation is applied at the last layer (e.g. APPNP~\citet{klicpera2018predict}).}


\begin{figure*}[t]
\centering
\begin{minipage}[t]{\linewidth}
\begin{minipage}[t]{0.33\linewidth}
\centering
\includegraphics[width=0.99\textwidth]{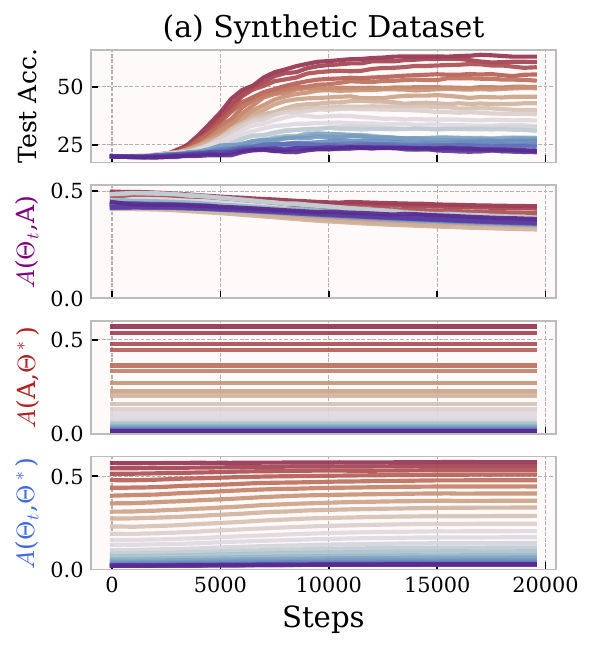}
\end{minipage}%
\begin{minipage}[t]{0.33\linewidth}
\centering
\includegraphics[width=0.99\textwidth]{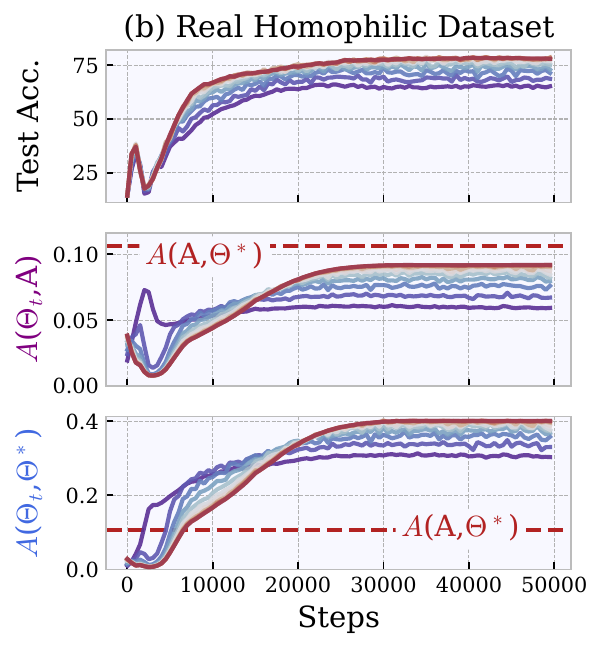}
\end{minipage}%
\begin{minipage}[t]{0.33\linewidth}
\centering
\includegraphics[width=0.99\textwidth]{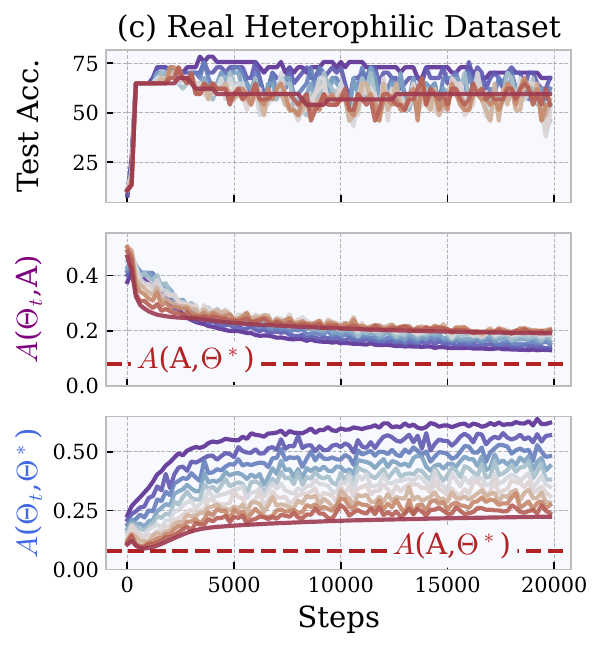}
\end{minipage}%
\end{minipage}
\caption{Evolution of NTK matrix $\mathbf \Theta_t$ of GCN during training, reflected by matrix alignment. \textbf{(a)} Synthetic dataset generated by a stochastic block model, where the homophily level gradually decreases by altering edge probabilities, i.e. {\color{Maroon}homophilic} $\rightarrow$ {\color{BlueViolet}heterophilic}; \textbf{(b \& c)} Real-world homophilic (\texttt{Cora}) and heterophilic (\texttt{Texas}) datasets, where the graph is gradually coarsened until there is no edge left when evaluating $\mathbf \Theta_t$, i.e. {\color{Maroon}more graph} $\rightarrow$ {\color{BlueViolet}less graph}. (Details in Appendix~\ref{app_detail2})}
\label{fig_theory}
\end{figure*}

\section{Generalization and Heterophily} \label{sec5}

In this section, we offer interpretable explanations of ``when and why GNNs successfully generalize" and their pathological training behavior on heterophilic graphs. We also study the time evolution of real-world GNN NTKs to further empirically verify our theoretical results.

\subsection{When and Why GNNs Generalize?} \label{sec_generalization}

Our previous discussions have revolved around two matrices, namely the graph adjacency matrix $\mathbf A$ and the NTK matrix $\mathbf \Theta_t$.\footnote{We here refer to $\mathbf A$ as a class of similarity matrices based on original $\mathbf A$ in a general sense, such as $\mathbf A^K$ etc.} To complete the theoretical picture, we introduce another matrix called the ideal \emph{or optimal kernel matrix}~\citep{cristianini2001kernel}, defined as $\mathbf \Theta^* \triangleq \bar{\mathbf Y} \bar{\mathbf Y}^\top \in \mathbb R^{n\times n}$ to indicate whether two instances have the same label, and a metric to quantify alignment of (kernel) matrices:
\begin{definition}[Alignment, \citet{cristianini2001kernel}] Given two (kernel) matrices $\mathbf K_1$ and $\mathbf K_2$, their alignment is defined as $A\left(\mathbf K_1, \mathbf K_2\right)\triangleq{\left \langle \mathbf K_1, \mathbf K_2\right\rangle_F}/{(\|\mathbf K_1\|_F \|\mathbf K_2\|_F)} \in[0,1]$. This is a generalization of cosine similarity from vectors to matrices, $\operatorname{arccos}$ of which satisfies the triangle inequality.
\end{definition}
\red{$\circ\;$ \emph{Homophily Level}: $A(\mathbf A, \mathbf \Theta^*)$}. The alignment between $\mathbf A$ and $\mathbf \Theta^*$ quantifies the \emph{homophily level} of graph structure, i.e. whether two connected nodes indeed have the same label, and is determined the dataset. While many empirical results (e.g. \citet{zhu2020beyond}) suggest high homophily level is important for the performance of GNNs, deeper theoretical understandings are mostly lacking.

\blue{$\circ\;$ \emph{Kernel-Target Alignment}: $A(\mathbf \Theta_t, \mathbf \Theta^*)$}. The alignment between kernel matrix and optimal $\mathbf \Theta^*$ has been widely studied and used as an objective for learning kernels functions~\citep{cristianini2001kernel,kwok2003learning,lanckriet2004learning,gonen2011multiple}. Better kernel-target alignment has been recognized as a critical factor that leads to favorable generalization for classic non-parametric models. For intuition in our case, one can quickly verify that substituting $\mathbf \Theta^*$ to the training dynamics in (\ref{eqn_inference}) leads to perfect generalization performance (since ground-truth labels only propagate to unseen instances with the same label).

\purple{$\circ\;$ \emph{Kernel-Graph Alignment}: $A(\mathbf \Theta_t, \mathbf A)$}. The alignment between NTK and graph is a novel notion in our work, as prior sections have shown that GNN NTK matrices naturally tend to align with $\mathbf A$. The RP algorithm (and variants thereof) serve as an extreme case with two identical matrices.

\textbf{Implications.} We consider two cases. For \emph{homophilic} graphs where \red{$A(\mathbf A, \mathbf \Theta^*) \uparrow$} is naturally large, better kernel-graph alignment \purple{$A(\mathbf \Theta_t, \mathbf A) \uparrow$} consequently leads to better kernel-target alignment \blue{$A(\mathbf \Theta_t, \mathbf \Theta^*) \uparrow$}. In other words, the NTK of GNNs naturally approaches the optimum as the graph structure possesses homophily property, and leveraging it in the optimization process (\ref{eqn_inference}) encourages training residuals to flow to unseen samples with the same label and thus better generalization; In contrast, for \emph{heterophilic} graphs where {$A(\mathbf A, \mathbf \Theta^*)$} is small, better kernel-graph alignment will hinder kernel-target alignment, explaining the pathological learning behavior of GNNs when dealing with heterophilic graphs in an interpretable manner. 

\subsection{Theoretical Results}

To support this interpretation, we examine the generalization behavior of infinitely-wide neural networks in the extreme case where its NTK matrix is governed by the graph, say, {\small$\lim _{k \rightarrow \infty} \sum_{i=0}^{k}(\alpha \mathbf A)^i$} with $\alpha \in (0,1)$ as adopted by the converged LP algorithm in (\ref{eqn_lp}). With a common assumption that training instances are drawn i.i.d.~from a distribution $\mathcal P$, and based on the Rademacher complexity generalization bound for kernel regression~\citep{bartlett2002rademacher,arora2019fine,du2019graph}, we have a label-dependent high-probability (at least $1-\delta$) upper bound on population risk (derivation in Appendix~\ref{app_gen_bound}) $\mathbb{E}_{(\boldsymbol x, y) \sim \mathcal{P}}\left[l\left(f(\boldsymbol x), y\right)\right] =$
\begin{equation} \label{eq_genbound}
\mathcal O\left(\sqrt{1 - c n_l^{-1} \red{A(\mathbf A, \mathbf \Theta^*)}}+\sqrt{n_l^{-1} \log (\delta^{-1})}\right),
\end{equation}
where $c = \alpha \|\mathbf \Theta^*\|_F \|\mathbf A\|_F$ is a constant, $A(\mathbf A, \mathbf \Theta^*)$ is the homophily level for the training set (with slight abuse of notation). This bound can also be viewed as a theoretical guarantee for the converged (generalized) RP algorithm and clearly demonstrates that:  \textit{for\purple{$A(\mathbf \Theta_t, \mathbf A)$} fixed as $1$, higher level of graph homophily \red{$A(\mathbf A, \mathbf \Theta^*)$} plays a dominant role in better generalization}. The above analysis can also be potentially extended to infinitely-wide GNNs discussed in Sec.~\ref{sec_gnndynamics} with further assumptions on the condition number of its NTK matrix, which will lead to similar (but less straightforward) results.

Pushing further, in a complementary setting to the above, where the objective is to find the optimal a priori kernel for directly minimizing the population risk (but without access to any input features or ground-truth labels), we demonstrate that when homophily assumptions are imposed on the graph, infinitely-wide GNNs and the RP algorithm will yield the optimal kernel regression predictor with the provably best generalization performance:

\begin{theorem} [Bayesian Optimality of GNN] \label{thm_gnnoptimal} 
Assume that the underlying data generation distribution $\mathcal P$ is such that the probability of a pair of instances having the same label $P(y(\boldsymbol x_i) = y(\boldsymbol x_j))$ is proportional to $\mathbf A_{ij} - {1}/{2}$. Then the optimal kernel regression predictor that minimizes the population risk with squared loss has kernel matrix $\mathbf A$.
\end{theorem}

\begin{remark}
See proof in Appendix~\ref{app_gnnoptimal}.
The matrix $\mathbf A$ from the above result could vary depending on different assumptions on $P(y(\boldsymbol x_i) = y(\boldsymbol x_j))$, such as {$\mathbf A^K$} for the basic version of RP, or $\mathbf{A}\left(\mathbf{A}^2 \odot \mathbf{S}\right) \mathbf{A}$ for infinitely-wide two-layer GNNs in Theorem~\ref{thm_twolayergnn}. And ideally, if one has privileged access to labels of all instances, this optimal matrix is exactly the optimal (posterior) kernel $\mathbf \Theta^*$ discussed above and perfect generalization will be achieved. This result further verifies our interpretation of generalization by showing that \emph{for \red{$A(\mathbf A, \mathbf \Theta^*)$} fixed to be large, better \purple{$A(\mathbf \Theta_t, \mathbf A)$} (e.g. GNN and RP) leads to favorable generalization.}
\end{remark}

\section{Empirical Verification} \label{sec_exp4theory}

We next empirically verify whether real-world GNN NTKs indeed align with the graph during training and its effect on generalization.
Figure~\ref{fig_theory} plots the results reflected by alignment between $\mathbf \Theta_t$, $\mathbf A$ and $\mathbf \Theta^*$.

\textbf{Synthetic Dataset.}\quad To demonstrate the effects of different homophily levels, we use stochastic block models~\citep{holland1983stochastic} to generate a synthetic dataset. As shown in Fig.~\ref{fig_theory}(a), the kernel-graph alignment \purple{$A(\mathbf \Theta_t, \mathbf A)$} for GNNs stays at a high level regardless of different graph structures, as a natural result of the network architecture (Sec.~\ref{sec_gnndynamics}). Consequently, as we gradually alter edge probabilities to decrease the homophily level \red{$A\left(\mathbf{A}, \mathbf{\Theta}^*\right)$}, the kernel-target alignment \blue{$A\left(\mathbf \Theta_t, \mathbf \Theta^*\right)$} also decreases and the testing accuracy drops dramatically (Sec.~\ref{sec_generalization}).

\textbf{Real-World Datasets.}\quad On real-world homophilic and heterophilic datasets, we progressively coarsen the graph until the GNN degrades to an MLP, allowing us to analyze the effects of feature propagation in the model. For now, let us first focus on comparing {\color{Maroon}red} and {\color{BlueViolet}blue} lines in Fig.~\ref{fig_theory}(b,c). We found the kernel-graph alignment \purple{$A(\mathbf \Theta_t, \mathbf A)$} overall decreases with less graph structure, again verifying results in Sec.~\ref{sec_gnndynamics}. However, the impact of the graph structure depends on different homophily levels \red{$A\left(\mathbf{A}, \mathbf{\Theta}^*\right)$}: on the homophilic dataset, more graph structure and better $A(\mathbf \Theta_t, \mathbf A)$ optimize the NTK matrix as reflected by better kernel-target alignment \blue{$A\left(\mathbf \Theta_t, \mathbf \Theta^*\right)$}, but worsens it on the heterophilic one. This results in distinct generalization behavior of GNNs reflected by the testing accuracy.

\textbf{Can and How (G)NNs Handle Heterophily?}\quad
Recently, there has been a growing discussion on whether standard GNNs are capable of handling heterophily, with empirical results pointing to diverging conclusions~\citep{zhu2020beyond,ma2021homophily,luan2022revisiting,platonov2023critical}. From the training dynamics perspective, we add new insights to this debate: 

\textbf{$\circ$ 1)} While we have shown that infinitely-wide GNNs are sub-optimal on heterophilic graphs, MLPs in contrast have no guarantee of kernel-target alignment (since both $A(\mathbf A, \mathbf \Theta^*)$ and $A(\mathbf \Theta_t, \mathbf A)$ are not well-aligned), indicating that \emph{they could either generalize better or worse without clearcut (dis)advantages, as opposed to the case on homophilic graphs where GNN is provably at an advantage}, which explains the diverse empirical results when comparing GNNs with MLPs in prior work. 

\textbf{$\circ$ 2)} As we now turn to the NTK's time evolution in Fig.~\ref{fig_theory}, an additional phenomenon we found across all datasets is the overall increase of kernel-target alignment \blue{$A\left(\mathbf \Theta_t, \mathbf \Theta^*\right)$} during the training process. This indicates that the training process enables real-world GNNs to adjust their NTK's feature space such that the kernel matrix leans towards an ultimately homoplilic graph structure (i.e. $\mathbf \Theta^*$) to adapt to heterophilic datasets (and consequently different evolutionary trends of \purple{$A\left(\mathbf \Theta_t, \mathbf A\right)$} for hemophiliac and heterophilic datasets). Such a phenomenon has also been found for other models in vision tasks~\citep{baratin2021implicit}.

Additionally, since our analysis also applies to other forms of $\mathbf A$ for feature propagation in GNNs, it could also potentially explain how some specialized models with different instantiations of $\mathbf A$ (e.g. signed propagation) can mitigate the heterophily issue. {We also refer readers to Appendix~\ref{app_heterophily} for a discussion of how our insights can be potentially used to guide designing new GNN architectures for handling heterophily.}

\section{Discussions} \label{sec_discussion}

\paragraph{Abridged Related Work.} Most existing work on theoretical aspects of GNNs focuses on representation and generalization (see \citet{jegelka2022theory} and references therein) while their optimization properties remains under-explored~\citep{zhang2020fast,xu2021optimization,yadati2022convex,huang2023graph}. Specifically, existing work in representation (or expressiveness) does not provide answers to what exactly GNN functions are found during the optimization process. For generalization, prior work is insufficient to explain the effects of training, which is widely-recognized as a crucial ingredient, and does not connect to heterophily, a relevant aspect in practice. (See a comprehensive discussion of related work in Appendix~\ref{app_relatedwork})

\textbf{Applicability and Future Directions.}\quad Our insights apply to both transductive and inductive settings (Appendix~\ref{app_induc_trans}), other loss functions (Appendix~\ref{app_loss}), multi-dimensional outputs (Appendix~\ref{app_multioutput}), different GNN (or graph Transformer) architectures that conform to our definition. The analytical framework could also be extended to other tasks (which are left as future work); for instances, for graph classification or regression, the definition of GNTK in (\ref{eqn_inference}) should be modified to the graph-level one~\citep{du2019graph}; for link prediction and self-supervised learning, the loss function in the derivation of (\ref{eqn_inference}) should be adjusted accordingly. A similar message passing process in function space would still hold in these settings. The framework could also be potentially adapted to analyze other GNN issues, such as over-smoothing and over-squashing, though specific considerations may be required. While this paper focuses specifically on studying the role of graph structures in optimization, the framework can potentially also be used to study the effects of different model architectures by noting that the definition of $\mathbf{A}$ in the paper is associated with GNN architectures (\ref{eqn_layer_gcn}). 


\paragraph{}

\section*{Acknowledgements}

The authors from SJTU were in part supported by NSFC (92370201, 62222607, 72342023). The author from CUHKSZ was supported by NSFC (12326608), the National Key Research and Development Project under grant 2022YFA1003900 and Hetao Shenzhen-Hong Kong Science and Technology Innovation Cooperation Zone Project (No.HZQSWS-KCCYB-2024016).

\section*{Impact Statement}

Theoretically, our analytical framework from the training dynamics perspective might have broader impacts beyond graph learning as it exemplifies using kernel learning to justify model architectures. And practically, the proposed RP also unlocks the possibility of leveraging sparse structures in training dynamics to develop efficient algorithms. Similar to other works in machine learning, this work may inevitably have negative society impacts if machine learning were to be employed for unethical purposes.

\bibliography{example_paper}
\bibliographystyle{icml2024}

\newpage
\appendix
\onecolumn


\section{Unabridged Related Work}\label{app_relatedwork}
This section discusses more related works that are not covered in the main text and those related works that are already covered but in greater depth.

\subsection{Optimization and Training Dynamics}
Towards deeper understanding of the success and limitation of GNNs, many works focus on the representation power of GNNs~\citep{maron2019provably,xu2018powerful,oono2019graph,chen2019equivalence,dehmamy2019understanding,sato2019approximation,loukas2020hard}. While these works formalize what functions a GNN can possibly represent, they do not provide answers to which specific GNN function will be found during optimization process or whether the learned GNN function will successfully generalize. In contrast, theoretical understandings of the optimization properties of GNNs are scarce. Specifically, \citet{zhang2020fast} prove the global convergence of a one-layer GNN with assumptions on the training algorithm built on tensor initialization and accelerated gradient descent; \citet{xu2021optimization} analyze the convergence rate of linearized GNNs with focus on training dynamics in weight space and empirical risk; \citet{yadati2022convex} analyze optimization properties of a two-layer GCN by introducing a convex program; \citet{huang2023graph} characterize signal learning and noise memorization in two-layer GCN and compare it to CNNs. However, none of existing studies study the training dynamics of GNNs in function space, which is our focus and could lead to a wealth of new insights that are theoretically and practically valuable. 



\subsection{In- and Out-of-Distribution Generalization}
Existing works in generalization of GNNs focus on their in-distribution~\citep{scarselli2018vapnik,verma2019stability,du2019graph,liao2020pac} and out-of-distribution generalization properties~\citep{yehudai2021local,xu2020neural,wu2022handling,yang2022graph}. For \emph{node-level tasks} which are challenging due to the dependency between samples, most prior art study the generalization bound of GNNs based on complexity of model class~\citep{scarselli2018vapnik,baranwal2021graph,garg2020generalization,ma2021subgroup} or algorithmic stability~\citep{verma2019stability,zhou2021generalization,cong2021provable}. The former line of works do not consider the optimization, which however is a critical ingredient of finding generalizable solutions; the latter line of works consider the training algorithm, but their bounds as the number of epochs becomes large. Moreover, there is no existing work (to the best of our knowledge) formally connecting generalization and heterophily, though their connections are tacitly implied in many empirical results, e.g.~\citep{zhu2020beyond,zhu2021graph,zheng2022graph}. The most related work in GNN generalization is~\citep{yang2022graph} wherein the authors found vanilla MLPs with test-time message passing operations can be as competitive as different GNN counterparts, and then analyzed generalization of GNNs using node-level GNTK with attention to feature-wise extrapolation. Our results agree with and complement \citep{yang2022graph} by noting that our analysis is applicable for all inductive, transductive and training without graph settings (cf. Appendix~\ref{app_induc_trans}).

\subsection{Graph-Based Semi-Supervised Learning (Label Propagation)} \label{app_rw_lp}
One of most popular type of semi-supervised learning methods is graph-based methods~\citep{chapelle2009semi}, where label propagation~\citep{szummer2001partially,zhu2002learning,zhu2003semi,zhou2003learning,chapelle2009semi,koutra2011unifying,gatterbauer2015linearized,yamaguchi2016camlp,iscen2019label} is one of the most classic and widely-used algorithms. The algorithm relies on external data structure, usually represented by a graph adjacency matrix $\mathbf A \in \{0,1\}^{n\times n}$, to propagate ground-truth labels of labeled samples to infer unlabeled ones, and it is still under active research nowadays, e.g.~\citep{pukdee2023label,lee2022ultraprop}, mainly due to its efficiency and scalability. The label propagation algorithm can usually be induced from minimization of a quadratic objective~\citep{zhou2003learning}:
\begin{equation}
\mathcal E =\|{\mathbf F} - \mathbf Y\|^2 + \lambda \operatorname{Tr}\left[[{\mathbf F}, \mathbf F']^\top (\mathbf I_n - \mathbf A) ~[{\mathbf F}, \mathbf F']\right], 
\end{equation}
motivated from the Dirichlet energy to enforce smoothness of predictions according to sample relations. Notice the minima of this objective does not minimize the squared loss part $\|{\mathbf F} - \mathbf Y\|^2$ due to the regularization of the second term. In contrast, the proposed RP algorithm has similar smoothness effects but minimizes the squared loss. Moreover, some previous works attempt to explore the interconnections between LP and GNNs from different perspectives such as feature/label influence~\citep{wang2021combining}, generative model~\citep{jia2022unifying}. Compared with them, we first show the exact equivalence of classification results between LP and the first step training of GNN (on node embeddings and with squared loss).

\section{Theoretical Results} \label{app_theory}
\subsection{Derivation of \eqref{eqn_nndynamics} and \eqref{eqn_inference}: Training Dynamics of General Parameterized Model} \label{proof_rp_ntk}
Recall that for supervised learning, one is interested in minimizing the squared error $\mathcal L$ using \emph{gradient descent (GD)},
\begin{equation} 
\mathcal L = \frac{1}{2}\Vert \mathbf F_t - \mathbf Y \Vert^2, \quad \frac{\partial \mathbf W_t}{\partial t} = - \eta \nabla_{\mathbf W} {\mathcal L}.
\end{equation}
Let $\mathbf R_t = \mathbf Y - \mathbf F_t$ and $\mathbf R'_t = \mathbf Y' - \mathbf F'_t$ denote residuals for the labeled and unlabeled sets. 

For the training set, we have
\begin{align}
&~~~~~~~~ \frac{\partial \mathbf F_t}{\partial t} \\
&=~~~ \frac{\partial \mathbf F_t}{\partial \mathbf W_t} ~ \frac{\partial \mathbf W_t}{\partial t} && \text{\textit{(Chain rule)}} \\
&=~~~ - \eta~ \frac{\partial \mathbf F_t}{\partial \mathbf W_t} ~ \nabla_{\mathbf W} {\mathcal L} && \text{\textit{(GD training)}}\\
&=~~~ -\eta~ \frac{\partial \mathbf F_t}{\partial \mathbf W_t} ~ \nabla_{\mathbf W} \mathbf F_t~ \nabla_{\mathbf F_t} \mathcal L && \text{\textit{(Chain rule)}}\\
&=~~~ -\eta~ {\nabla_{\mathbf W} \mathbf F_t}^\top ~ {\nabla_{\mathbf W} \mathbf F_t}~ {\nabla_{\mathbf F_t} \mathcal L} && \text{\textit{(Change of notation)}} \vphantom{\frac{\partial \mathbf R_t}{\partial \mathbf W_t}} \\
&=~~~ -{\eta}~ \mathbf \Theta_t(\mathbf X, \mathbf X)~ {\nabla_{\mathbf F_t} \mathcal L} && \text{\textit{(Neural tangent kernel)}} \vphantom{\frac{\partial \mathbf R_t}{\partial \mathbf W_t}} \\
&=~~~ \eta ~  \mathbf \Theta_t(\mathbf X, \mathbf X)~ {\mathbf R_t} && \text{\textit{(Loss function)}} \vphantom{\frac{\partial \mathbf R_t}{\partial \mathbf W_t}}
\end{align}

Incorporating the test set, we have
\begin{align}
&~~~~~~~~ \frac{\partial [\mathbf R_t, \mathbf R'_t]}{\partial t} \\
&=~~~ \frac{\partial[\mathbf R_t, \mathbf R'_t]}{\partial \mathbf W_t} ~ \frac{\partial \mathbf W_t}{\partial t} && \text{\textit{(Chain rule)}} \\
&=~~~ - \eta~ \frac{\partial[\mathbf R_t, \mathbf R'_t]}{\partial \mathbf W_t} ~ \nabla_{\mathbf W} {\mathcal L} && \text{\textit{(GD training)}}\\
&=~~~ -\eta~ \frac{\partial [\mathbf R_t, \mathbf R'_t]}{\partial \mathbf W_t} ~ \nabla_{\mathbf W} [\mathbf F_t, \mathbf F'_t]~ \nabla_{[\mathbf F_t, \mathbf F'_t]} \mathcal L && \text{\textit{(Chain rule)}}\\
&=~~~ \eta~ {\nabla_{\mathbf W} [\mathbf F_t, \mathbf F'_t]}^\top ~ {\nabla_{\mathbf W} [\mathbf F_t, \mathbf F'_t]}~ {\nabla_{[\mathbf F_t, \mathbf F'_t]} \mathcal L} && \text{\textit{(Change of notation)}} \vphantom{\frac{\partial \mathbf R_t}{\partial \mathbf W_t}} \\
&=~~~ {\eta}~ \mathbf \Theta_t([\mathbf X, \mathbf X'], [\mathbf X, \mathbf X'])~ {\nabla_{[\mathbf F_t, \mathbf F'_t]} \mathcal L} && \text{\textit{(Neural tangent kernel)}} \vphantom{\frac{\partial \mathbf R_t}{\partial \mathbf W_t}} \\
&=~~~ {\eta}~ \mathbf \Theta_t([\mathbf X, \mathbf X'], [\mathbf X, \mathbf X'])~ {\nabla_{[\mathbf F_t, \mathbf F'_t]} \frac{1}{2}\Vert \mathbf F_t - \mathbf Y \Vert_F^2} && \text{\textit{(Loss function)}} \vphantom{\frac{\partial \mathbf R_t}{\partial \mathbf W_t}} \\
&=~~~ -\eta ~  \mathbf \Theta_t([\mathbf X, \mathbf X'], [\mathbf X, \mathbf X'])~ [{\mathbf R_t}, \mathbf 0] && \text{\textit{(Compute gradient)}} \vphantom{\frac{\partial \mathbf R_t}{\partial \mathbf W_t}}
\end{align}
where 
\begin{equation}
\mathbf \Theta_t([\mathbf X, \mathbf X'], [\mathbf X, \mathbf X']) \triangleq \left[\begin{array}{cc}
\mathbf \Theta_t(\mathbf X, \mathbf X) & \mathbf \Theta_t(\mathbf X, \mathbf X') \\
\mathbf \Theta_t(\mathbf X', \mathbf X) & \mathbf \Theta_t(\mathbf X', \mathbf X')
\end{array}\right]
\end{equation}
is the NTK matrix in $\mathbb R^{n\times n}$. Discretizing the residual dynamics with step size $\Delta t = 1$ and rearranging the equation gives
\begin{equation}
\left[\mathbf R_{t+1}, \mathbf R'_{t+1}\right] = -\eta~ \left[\begin{array}{cc}
\mathbf \Theta_t(\mathbf X, \mathbf X) & \mathbf \Theta_t(\mathbf X, \mathbf X') \\
\mathbf \Theta_t(\mathbf X', \mathbf X) & \mathbf \Theta_t(\mathbf X', \mathbf X')
\end{array}\right]  [\mathbf R_t, \mathbf 0] ~~+~~ \left[\mathbf R_t, \mathbf R'_{t}\right].
\end{equation}
or
\begin{equation}
\left[\mathbf R_{t+1}, \mathbf R'_{t+1}\right] = -\eta~ \mathbf \Theta_t(\bar{\mathbf X}, \bar{\mathbf X}) [\mathbf R_t, \mathbf 0] ~~+~~ \left[\mathbf R_t, \mathbf R'_{t}\right].
\end{equation}

\subsection{Proof of Theorem~\ref{thm_converge}: Convergence of RP, Connection with Kernel Regression} \label{app_converge}
In this proof we examine generic RP iterations of the form given by
\begin{equation} \label{appeq_rp}
\left[\mathbf R_{t+1}, \mathbf R'_{t+1}\right] =  -\eta \mathbf S   [\mathbf R_t, \mathbf 0]  + \left[\mathbf R_t, \mathbf R'_{t}\right],
\end{equation}
where $\mathbf S$ is an arbitrary symmetric matrix as similarity measure ($\mathbf S = \mathbf A^K$ for the basic version of RP in (\ref{eqn_rp})) with block structure aligned with the dimensions of $\mathbf R_{t}$ and $\mathbf R'_{t}$ respectively
\begin{equation}
\mathbf S = \left[\begin{array}{cc} \mathbf S_{\mathbf X \mathbf X} & \mathbf S_{\mathbf X \mathbf X'} \\ \mathbf S_{\mathbf X' \mathbf X} & \mathbf S_{\mathbf X' \mathbf X'} \end{array} \right],
\end{equation}
where $\mathbf S_{\mathbf X\mathbf X} \in \mathbb R^{n_l \times n_l}$ is a principal submatrix of $\mathbf S\in\mathbb R^{n\times n}$ corresponding to the training set. 
According to (\ref{appeq_rp}), the explicit form of training residuals can be written as
\begin{eqnarray} \label{appeq_train_residual}
\mathbf R_{t+1} &=& \left(\mathbf I_{n_l} -\eta \mathbf S_{\mathbf X \mathbf X} \right) \mathbf R_t \nonumber\\
&=& \left(\mathbf I_{n_l} -\eta \mathbf S_{\mathbf X \mathbf X} \right)^{t+1} \mathbf Y
\end{eqnarray}
and testing residuals can be written as 
\begin{eqnarray}\label{appeq_test_residual}
\mathbf R'_{t+1} &=& -\eta \mathbf S_{\mathbf X'\mathbf X} \mathbf R_t + \mathbf R'_t = - \eta \sum_{i=0}^{t} \mathbf S_{\mathbf X' \mathbf X} \mathbf R_i \nonumber\\
&=& - \eta \mathbf S_{\mathbf X' \mathbf X} \sum_{i=0}^{t}  \left(\mathbf I_{n_l} -\eta \mathbf S_{\mathbf X \mathbf X} \right)^{i} \mathbf Y
\end{eqnarray}
To analyze their convergence, we consider three variants based on the property of $\mathbf S_{\mathbf X\mathbf X}$: 1) $\mathbf S_{\mathbf X\mathbf X}$ is positive definite; 2) $\mathbf S_{\mathbf X\mathbf X}$ is positive semi-definite (but not positive definite); 3) $\mathbf S_{\mathbf X\mathbf X}$ is not positive semi-definite. 

For the first variant, we stipulate $\mathbf S_{\mathbf X \mathbf X}$ is positive definite. In this case, sufficiently small step size $\eta < 2/\sigma_{max}[\mathbf S_{\mathbf X \mathbf X}]$, where $\sigma_{max}[\mathbf S_{\mathbf X \mathbf X}]$ is the largest eigenvalue of $\mathbf S_{\mathbf X \mathbf X}$, ensures each diagonal element in $\mathbf I_{n_l} -\eta \mathbf S_{\mathbf X \mathbf X}$ to lie between $(-1,1)$. Therefore, as $t\rightarrow \infty$, the power and geometric series of $\mathbf I_{n_l} -\eta \mathbf S_{\mathbf X \mathbf X}$ converge to
\begin{equation}
\begin{split}
\left(\mathbf I_{n_l} -\eta \mathbf S_{\mathbf X \mathbf X} \right)^{t+1} &\rightarrow \mathbf Q \mathbf (\mathbf I_{n_l} - \eta \mathbf \Lambda[\mathbf S_{\mathbf X \mathbf X}])^{\infty} \mathbf Q^{-1} = \mathbf 0,\\
\sum_{i=0}^{t}  \left(\mathbf I_{n_l} -\eta \mathbf S_{\mathbf X \mathbf X} \right)^{i} &\rightarrow (\mathbf I_{n_l} - (\mathbf I_{n_l} - \eta \mathbf S_{\mathbf X \mathbf X}))^{-1} = (\eta \mathbf S_{\mathbf X \mathbf X})^{-1}.
\end{split}
\end{equation}
Plugging them back into (\ref{appeq_train_residual}) and (\ref{appeq_test_residual}) gives us
\begin{equation}
\begin{split}
\mathbf R_t \rightarrow \mathbf 0,\quad \mathbf R'_t \rightarrow  -\mathbf S_{\mathbf X' \mathbf X} \mathbf S_{\mathbf X \mathbf X}^{-1} \mathbf Y.
\end{split}
\end{equation}
Correspondingly, by noting that $\mathbf Y' = \mathbf 0$ at initialization, we have 
\begin{equation} \label{appeq_converge_pd}
\begin{split}
\mathbf F_t \rightarrow \mathbf Y, \quad \mathbf F'_t \rightarrow  \mathbf S_{\mathbf X'\mathbf X} \mathbf S_{\mathbf X\mathbf X}^{-1} \mathbf Y.
\end{split}
\end{equation}
Namely, for positive definite $\mathbf S_{\mathbf X\mathbf X}$, the converged model predictions $\mathbf F_{\infty}$ perfectly fit the training labels $\mathbf Y$, and we further assume that $\mathbf S$ is positive definite, \eqref{appeq_converge_pd} is equivalent to the solution of kernel regression with respect to the kernel $\kappa(\boldsymbol x, \boldsymbol x') = \mathbf A^K_{\boldsymbol x \boldsymbol x'}$.

For the second variant, we stipulate that $\mathbf S_{\mathbf X \mathbf X}$ is positive semi-definite (but not positive definite), which implies that $\mathbf S_{\mathbf X \mathbf X} = \mathbf B \mathbf B^\top$ for some tall matrix $\mathbf B\in \mathbb R^{n_l\times r}$ that is full column rank. Suppose $\mathbf P_{C(\mathbf B)} = \mathbf B(\mathbf B^\top \mathbf B)^{-1} \mathbf B^\top$ and $\mathbf P_{N(\mathbf B^\top)} = \mathbf I_{n_l} - \mathbf P_{C(\mathbf B)}$ denotes the projection matrices onto the column space of $\mathbf B$ and null space of $\mathbf B^\top$. 
By their construction,
\begin{equation}
\mathbf P_{C(\mathbf B)} \mathbf B = \mathbf B, \quad \mathbf P_{N(\mathbf B^\top)} \mathbf B = \mathbf 0.
\end{equation}

For training residuals, we may form the decomposition
\begin{eqnarray} \label{appeq_train_rp_converge}
    \mathbf R_{t+1} &=& (\mathbf I_{n_l} - \eta\mathbf B\mathbf B^\top)^{t+1} \mathbf Y \nonumber\\
    &=& (\mathbf I_{n_l} - \eta\mathbf B\mathbf B^\top)^{t+1} \mathbf P_{ C(\mathbf B)} \mathbf Y +(\mathbf I_{n_l} - \eta\mathbf B\mathbf B^\top)^{t+1} \mathbf P_{N(\mathbf B^\top)}\mathbf Y
\end{eqnarray}
It follows that, for the first term on the RHS of (\ref{appeq_train_rp_converge}), 
\begin{eqnarray}
    (\mathbf I_{n_l} - \eta\mathbf B\mathbf B^\top)^{t+1} \mathbf P_{ C(\mathbf B)} \mathbf Y &=&  (\mathbf I_{n_l} - \eta\mathbf B\mathbf B^\top)^{t+1} \mathbf B \mathbf B^\dagger  \mathbf Y \nonumber\\
    &=&  (\mathbf I_{n_l} - \eta\mathbf B\mathbf B^\top)^{t} (\mathbf I_{n_l} - \eta\mathbf B\mathbf B^\top) \mathbf B \mathbf B^\dagger  \mathbf Y \nonumber\\
    &=&  (\mathbf I_{n_l} - \eta\mathbf B\mathbf B^\top)^{t} \mathbf B (\mathbf I_r - \eta\mathbf B^\top\mathbf B)  \mathbf B^\dagger  \mathbf Y \nonumber\\
    &=&  \mathbf B (\mathbf I_r - \eta\mathbf B^\top\mathbf B)^{t+1}  \mathbf B^\dagger  \mathbf Y \nonumber\\
    &\rightarrow& \mathbf 0,
\end{eqnarray}
where $\mathbf B^\dagger = (\mathbf B^\top \mathbf B)^{-1} \mathbf B^\top$ is the pseudo inverse of $\mathbf B$. The convergence also requires $\eta < 2/\sigma_{max}[\mathbf S_{\mathbf X \mathbf X}]$. 

For the second term on the RHS of (\ref{appeq_train_rp_converge}), we have
\begin{eqnarray}
    (\mathbf I - \eta\mathbf B\mathbf B^\top)^{t+1} \mathbf P_{N(\mathbf B^\top)} \mathbf Y &=&  (\mathbf I - \eta\mathbf B\mathbf B^\top)^{t} (\mathbf P_{ N(\mathbf B^\top)} - \eta \mathbf B  (\mathbf P_{N(\mathbf B^\top)} \mathbf B)^\top)  \mathbf Y \nonumber \\
    &=&  (\mathbf I - \eta\mathbf B\mathbf B^\top)^{t} \mathbf P_{N(\mathbf B^\top)}\mathbf Y \nonumber\\
    &=&  \mathbf P_{N(\mathbf B^\top)}\mathbf Y
\end{eqnarray}

It follows that
\begin{equation}
\mathbf R_t \rightarrow \mathbf P_{N(\mathbf B^\top)}\mathbf Y = \left(\mathbf I_{n_l} - \mathbf B (\mathbf B^\top \mathbf B)^{-1} \mathbf B^\top \right) \mathbf Y,
\end{equation}
which is equivalent to optimal training residuals of linear regression~\citep{boyd2004convex}.

However, in our case, the testing residuals will not necessarily converge for arbitrary $\mathbf S_{\mathbf X'\mathbf X}$ since the geometric series in (\ref{appeq_test_residual}) diverges outside $(-1,1)$. Nevertheless, we can further stipulate $\mathbf S$ is also positive semi-definite. In this case, we have $\mathbf S_{\mathbf X' \mathbf X} = \mathbf B' \mathbf B^\top$ for another tall matrix $\mathbf B' \in \mathbb R^{n_u \times r}$. Correspondingly, the testing residuals can be written as
\begin{eqnarray}
    \mathbf R'_{t+1} &=& - \eta \mathbf B' \mathbf B^\top \sum_{i=0}^{t}  \left(\mathbf I_{n_l} -\eta \mathbf B \mathbf B^\top \right)^{i} \mathbf Y \nonumber\\
    &=& - \eta \mathbf B' \sum_{i=0}^{t}  \left(\mathbf I_r -\eta \mathbf B^\top \mathbf B \right)^{i} \mathbf B^\top \mathbf Y \nonumber\\
    &=& - \eta \mathbf B' \sum_{i=0}^{t}  \left(\mathbf I_r -\eta \mathbf B^\top \mathbf B \right)^{i} \mathbf B^\top \mathbf Y \nonumber\\
    &\rightarrow& -\eta\mathbf B' (\mathbf I_r - (\mathbf I_r -\eta\mathbf B^\top \mathbf B))^{-1} \mathbf B^\top \mathbf Y \nonumber\\
    &=& -\mathbf B' (\mathbf B^\top \mathbf B)^{-1} \mathbf B^\top \mathbf Y.
\end{eqnarray}

For the last variant, if $\mathbf S_{\mathbf X \mathbf X}$ is not positive semi-definite, it is no longer possible to guarantee convergence from arbitrary initializations. Rather we can only establish that solutions of the form described above can serve as fixed points of the iterations.

\begin{table}[t!]
\centering
\caption{Summary of the convergence of residual propagation. The convergence of model predictions can be inferred by $\mathbf F_t = \mathbf Y - \mathbf R_t$ and $\mathbf F'_t = - \mathbf R'_t$.} \label{tab_converge}
\resizebox{1.0\textwidth}{!}{
\setlength{\tabcolsep}{5mm}{
\begin{tabular}{@{}c|c|c|c|c@{}}
\toprule
$\mathbf S_{\mathbf X\mathbf X}$ & $\mathbf S$ & Convergence of $\mathbf R_t$ & Convergence of $\mathbf R'_t$ & Counterpart \\ \midrule 
 PD &  PSD  & $\mathbf 0$ & $-\mathbf S_{\mathbf X' \mathbf X} \mathbf S_{\mathbf X \mathbf X}^{-1} \mathbf Y$ & Kernel Regression  \\ 
 PD & not PSD & $\mathbf 0$ & $-\mathbf S_{\mathbf X' \mathbf X} \mathbf S_{\mathbf X \mathbf X}^{-1} \mathbf Y$ & Unique   \\ 
 PSD (but not PD) & PSD &  $\left(\mathbf I_{n_l} - \mathbf B (\mathbf B^\top \mathbf B)^{-1} \mathbf B^\top \right) \mathbf Y$ & $-\mathbf B' (\mathbf B^\top \mathbf B)^{-1} \mathbf B^\top \mathbf Y$ & Linear Regression  \\
 PSD (but not PD) & not PSD &  $\left(\mathbf I_{n_l} - \mathbf B (\mathbf B^\top \mathbf B)^{-1} \mathbf B^\top \right) \mathbf Y$ & Not Converge & Unique  \\
 not PSD & not PSD&  Not Converge  &  Not Converge & Unique \\\bottomrule
\end{tabular}}}
\end{table}

To summarize, for the RP algorithm in (\ref{eqn_rp}) where $\mathbf S = \mathbf A^K$, if step size is sufficiently small $\eta < 2 \sigma^{-1}_{max}[\mathbf A^K_{\mathbf X\mathbf X}]$, both $\mathbf R_t$ and $\mathbf R'_t$ converge as $t\rightarrow \infty$ for positive definite $\mathbf A^K_{\mathbf X\mathbf X}$ or positive semi-definite $\mathbf A^K$. In practice, one can choose $K$ as an even number or letting $\mathbf A \leftarrow \alpha \mathbf A + (1-\alpha) \mathbf I_n$ for $\alpha>=\frac{1}{2}$ to enforce positive semi-definiteness of $\mathbf A^K$. 
However, in practice, we did not find that test performance was compromised when the algorithm does not converge. Table~\ref{tab_converge} gives a clear overview of the convergence of RP and connection with existing learning algorithms.

\subsection{Computation of Node-Level GNTK} \label{app_gntk}

We present the recurrent formula for computing of GNTK in node-level tasks. 
Suppose the GNN is denoted as $f(\boldsymbol x_i; \mathbf A)\in \mathbb R$ where the input is node feature $\boldsymbol x_i$, the graph structure $\mathbf A$ is used for cross-instance feature propagation at each layer, weights are $\mathbf W$. The GNTK in node-level tasks is defined as
\begin{equation}
\mathbf \Theta_t(\boldsymbol x_i, \boldsymbol x_j; \mathbf A) = \left\langle\frac{\partial f_t(\boldsymbol x_i; \mathbf A)}{\partial \mathbf W_t}, \frac{\partial f_t\left(\boldsymbol x_j; \mathbf A\right)}{\partial \mathbf W_t}\right\rangle,
\end{equation}
for a pair of nodes (i.e. data points) at optimization time index $t$. Intuitively, the kernel function measures quantifies similarity between seen and unseen instances based on how differently their outputs given by the GNN change by an infinitesimal perturbation of weights (which we will show is biased by the adjacency matrix used for feature propagation in GNNs).

Note that we consider transductive learning here, which is more convenient. For inductive learning, one should replace $\bar{\mathbf X}$ with $\mathbf X$, and $\mathbf A$ with its submatrix $\mathbf A_{\mathbf X\mathbf X}$. Let us denote GNTK with/without feature propagation at $\ell$-th layer as ${\mathbf \Theta}^{(\ell)}$/$\bar{\mathbf \Theta}^{(\ell)}$, covariance matrix of outputs of $\ell$-th layer with/without feature propagation as ${\mathbf \Sigma}^{(\ell)} / \bar{\mathbf \Sigma}^{(\ell)}$, covariance matrix of the derivative to $\ell$-th layer layer as $\dot{{\mathbf \Sigma}}^{(\ell)}$. The recurrent formula for computing node-level GNTK for infinitely-wide GNNs can be written as the following. The initialization of GNTK for graph regression is given as
\begin{equation}
\begin{split}
        \bar{\mathbf{\Theta}}^{(1)}(\boldsymbol x_i, \boldsymbol x_j; \mathbf A)& =\bar{\boldsymbol{\Sigma}}^{(1)}\left(\boldsymbol x_i, \boldsymbol x_j; \mathbf A\right) = \boldsymbol x_i^\top \boldsymbol x_{j},
\end{split}
\end{equation}

We also write matrix form of computing GNTK here in order to give clearer intuitions of how the adjacency matrix can be naturally encoded into the computation of $\mathbf \Theta$ in node regression tasks:
\begin{equation}
\begin{split}
        \bar{\mathbf{\Theta}}^{(1)}(\bar{\mathbf X}, \bar{\mathbf X}; \mathbf A)& =\bar{\boldsymbol{\Sigma}}^{(1)}\left(\bar{\mathbf X}, \bar{\mathbf X}; \mathbf A\right) = \bar{\mathbf X}^\top \bar{\mathbf X}.
\end{split}
\end{equation}

\paragraph{Propagation.}  The feature propagation operation (i.e. $\mathbf Z\gets \mathbf A \mathbf Z$) at each layer corresponds to
\begin{equation}
\begin{split}
    \boldsymbol{\Sigma}^{(\ell-1)}\left(\boldsymbol x_i, \boldsymbol x_j; \mathbf A\right)&= \sum_{i' \in \mathcal{N}_i} \sum_{j^{\prime} \in \mathcal{N}_{j}} \mathbf A_{ii'} \mathbf A_{jj'} \bar{\boldsymbol{\Sigma}}^{(\ell-1)}\left(\boldsymbol x_{i'}, \boldsymbol x_{j'}; \mathbf A\right)\\
    \mathbf{\Theta}^{(\ell-1)}\left(\boldsymbol x_i, \boldsymbol x_j; \mathbf A\right)&= \sum_{i' \in \mathcal{N}_i } \sum_{j^{\prime} \in \mathcal{N}_{j}} \mathbf A_{ii'} \mathbf A_{jj'} \bar{\mathbf{\Theta}}^{(\ell-1)}\left(\boldsymbol x_{i'}, \boldsymbol x_{j'}; \mathbf A\right),
\end{split}
\end{equation}

where $\mathcal N_i$ denote neighboring nodes of $\boldsymbol x_i$ including $\boldsymbol x_i$ itself. For GCN where $\mathbf A$ is defined as a symmetric normalized adjacency matrix, we have $\mathbf A_{ij} = \mathbf A_{ji} = 1/{\sqrt{d_i  d_{j}}}$. The above equation can be compactly described in a matrix form:
\begin{equation}
\begin{split}
    \boldsymbol{\Sigma}^{(\ell-1)}\left(\bar{\mathbf X}, \bar{\mathbf X}; \mathbf A\right)&=  \mathbf A~ \bar{\boldsymbol{\Sigma}}^{(\ell-1)}\left(\bar{\mathbf X}, \bar{\mathbf X}; \mathbf A\right) \mathbf A\\
    \mathbf{\Theta}^{(\ell-1)}\left(\bar{\mathbf X}, \bar{\mathbf X}; \mathbf A\right)&=  \mathbf A~\bar{\mathbf{\Theta}}^{(\ell-1)}\left(\bar{\mathbf X}, \bar{\mathbf X}; \mathbf A\right) \mathbf A.
\end{split}
\end{equation}

The multiplication of $\mathbf A$ before $\bar{\boldsymbol{\Theta}}^{(\ell-1)}\left(\bar{\mathbf X}, \bar{\mathbf X}; \mathbf A\right)$ gives row-wise weighted summation and after $\bar{\boldsymbol{\Theta}}^{(\ell-1)}\left(\bar{\mathbf X}, \bar{\mathbf X}; \mathbf A\right)$ gives column-wise weighted summation. 


\paragraph{Transformation.} Let $\mathcal{T}$ and $\dot{\mathcal{T}}$ be functions from $2 \times 2$ positive semi-definite matrices $\mathbf \Lambda$ to $\mathbb{R}$ given by
\begin{equation}
\left\{\begin{array}{l}\mathcal{T}(\mathbf \Lambda)=\mathbb{E}[\sigma(a) \sigma(b)] \\ \dot{\mathcal{T}}(\mathbf \Lambda)=\mathbb{E}\left[\sigma^{\prime}(a) \sigma^{\prime}(b)\right]\end{array} \quad(a, b) \sim \mathcal{N}(0, \mathbf \Lambda)\right.
\end{equation}

where $\sigma'$ is the derivative of the activation function. Then, the feature transformation process (i.e. $\mathbf Z = \sigma(\mathbf Z \mathbf W)$) then corresponds to:
\begin{equation}
\begin{gathered}
{\bar{\boldsymbol{\Sigma}}^{(\ell)}\left(\boldsymbol x_i, \boldsymbol x_j;\mathbf A\right)=c_\sigma \mathcal{T}\left(\begin{array}{ll}
{\boldsymbol{\Sigma}^{(\ell-1)}(\boldsymbol x_i, \boldsymbol x_i;\mathbf A)} & {\boldsymbol{\Sigma}^{(\ell-1)}\left(\boldsymbol x_i, \boldsymbol x_j;\mathbf A\right)} \\
{\boldsymbol{\Sigma}^{(\ell-1)}\left(\boldsymbol x_j, \boldsymbol x_i;\mathbf A\right)} & {\boldsymbol{\Sigma}^{(\ell-1)}\left(\boldsymbol x_j, \boldsymbol x_j;\mathbf A\right)}
\end{array}\right),} \\
{\dot{\boldsymbol{\Sigma}}^{(\ell)}\left(\boldsymbol x_i, \boldsymbol x_j;\mathbf A\right)=c_\sigma \dot{\mathcal{T}}\left(\begin{array}{ll}
{\boldsymbol{\Sigma}^{(\ell-1)}(\boldsymbol x_i, \boldsymbol x_i;\mathbf A)} & {\boldsymbol{\Sigma}^{(\ell-1)}\left(\boldsymbol x_i, \boldsymbol x_j;\mathbf A\right)} \\
{\boldsymbol{\Sigma}^{(\ell-1)}\left(\boldsymbol x_j, \boldsymbol x_i;\mathbf A\right)} & {\boldsymbol{\Sigma}^{(\ell-1)}\left(\boldsymbol x_j, \boldsymbol x_j;\mathbf A\right)}
\end{array}\right).}
\end{gathered}
\end{equation}

The layer-wise computation for $\mathbf \Theta$ is given as
\begin{equation}
\begin{split}
\bar{\mathbf{\Theta}}^{(\ell)}\left(\bar{\mathbf X}, \bar{\mathbf X};\mathbf A\right)&=\mathbf{\Theta}^{(\ell-1)}\left(\bar{\mathbf X}, \bar{\mathbf X};\mathbf A\right)\odot\dot{\boldsymbol{\Sigma}}^{(\ell)}\left(\bar{\mathbf X}, \bar{\mathbf X};\mathbf A\right)+\bar{\mathbf{\Sigma}}^{(\ell)}\left(\bar{\mathbf X}, \bar{\mathbf X};\mathbf A\right),
\end{split}
\end{equation}

where $\odot$ denotes Hadamard product.
For a $\ell$-layer GNN, the corresponding GNTK is given by $\mathbf \Theta^{(\ell)}(\bar{\mathbf X}, \bar{\mathbf X};\mathbf A)$, which is abbreviated as $\mathbf \Theta^{(\ell)}$ in the main text.

\subsection{Proof of Theorem~\ref{thm_twolayergnn}: Two-Layer GNN} \label{sec_twolayergnn}
Consider a two-layer GNN defined as 
\begin{equation}
\{f(\boldsymbol x;\mathbf A)\}_{\boldsymbol x\in \bar{{\mathbf X}}} = [\mathbf F, \mathbf F']=\mathbf{A} \frac{1}{\sqrt{m}} \sigma  \left(\mathbf{A} \bar{\mathbf X} \mathbf{W}^{(1)}\right) \mathbf{W}^{(2)}
\end{equation}
where $m$ is the width and $\sigma$ is ReLU activation. Following previous works that analyze two-layer fully-connected neural networks~\citep{arora2019fine,du2019gradient}, we consider optimization of the first layer weights $\mathbf{W}^{(1)}$ while fixing the second-layer weights. In this case, the two-layer GNTK is defined as
\begin{equation} 
\mathbf \Theta^{(2)}(\boldsymbol x_i, \boldsymbol x_j; \mathbf A) = \left\langle \frac{\partial f(\boldsymbol x_i ; \mathbf{A}, \mathbf{W})}{\partial \mathbf W^{(1)}}, \frac{\partial f(\boldsymbol x_j ; \mathbf{A}, \mathbf{W})}{\partial \mathbf W^{(1)}}\right\rangle.
\end{equation}
We can derive the the explicit form formula for computing it in overparameterized regime based on the general formula for arbitrarily deep GNNs given in Appendix.~\ref{app_gntk} (or directly calculating ${\partial f(\boldsymbol x_i ; \mathbf{A}, \mathbf{W})}/{\partial \mathbf W^{(1)}}$ by using chain rule). Consequently, we have
\begin{equation}
\begin{split}
\mathbf \Theta^{(2)}(\boldsymbol x_i, \boldsymbol x_j; \mathbf A) = &\sum_{{i'} \in \mathcal{N}(i)} \sum_{{j'} \in \mathcal{N}(j)} \mathbf A_{ii'} \mathbf A_{jj'} \left([\mathbf A\bar{\mathbf X}]_{i'}^{\top} [\mathbf A\bar{\mathbf X}]_{j'}\right)\\
&\mathbb{E}_{\boldsymbol{w} \sim \mathcal{N}(0, 1)}\left[\mathds{1}\left\{\boldsymbol{w}^{\top} [\mathbf A\bar{\mathbf X}]_{i'} \geq 0, \boldsymbol{w}^{\top} [\mathbf A\bar{\mathbf X}]_{j'} \geq 0\right\}\right]
\end{split}
\end{equation}
where $\mathcal N(i)$ denote the set of neighboring nodes, $\boldsymbol w \in \mathbb R^d$ is sampled from Gaussian distribution, and $\mathds{1}$ is indicator function. 

\paragraph{Case 1.} If we stipulate the input features are represented by an identity matrix, i.e. $\bar{\mathbf X} = \mathbf I_n$, it follows that
\begin{equation}
\begin{split}
\mathbf \Theta^{(2)}(\boldsymbol x_i, \boldsymbol x_j; \mathbf A) &= \sum_{{i'} \in \mathcal{N}(i)} \sum_{{j'} \in \mathcal{N}(j)} \mathbf A_{ii'} \mathbf A_{jj'} \left(\mathbf A_{i'}^{\top} \mathbf A_{j'}\right)\mathbb{E}_{\boldsymbol{w} \sim \mathcal{N}(0, 1)}\left[\mathds{1}\left\{\boldsymbol{w}^{\top} \mathbf A_{i'} \geq 0, \boldsymbol{w}^{\top} \mathbf A_{j'} \geq 0\right\}\right]\\
&= \sum_{{i'} \in \mathcal{N}(i)} \sum_{{j'} \in \mathcal{N}(j)} \mathbf A_{ii'} \mathbf A_{jj'} \frac{\mathbf A_{i'}^{\top} \mathbf A_{j'}(\pi - \operatorname{arccos}(\frac{\mathbf A_{i'}^\top \mathbf A_{j'}}{\|\mathbf A_{i'}\| \|\mathbf A_{j'}\|}))}{2\pi}.
\end{split}
\end{equation}
The above equation can be written neatly in the matrix form
\begin{equation}
\mathbf \Theta^{(2)}(\bar{\mathbf X}, \bar{\mathbf X}; \mathbf A) = \mathbf A (\mathbf A^2 \odot {\mathbf S})\mathbf A, \text{\quad where\quad}{\mathbf S}_{ij} = {(\pi - \operatorname{arccos}(\frac{\mathbf A_i^\top \mathbf A_j}{\|\mathbf A_i\| \|\mathbf A_j\|}))}/{2\pi}.
\end{equation}
The matrix $\mathbf S$ reweights each entry in $\mathbf A^2$ by the similarity of neighborhood patterns of $\boldsymbol x_i$ and $\boldsymbol x_j$.
Substituting it to (\ref{eqn_inference}) gives us the training dynamics of overparameterized two-layer GNN
\begin{equation}
\begin{split}
\left[\mathbf R_{t+1}, \mathbf R'_{t+1}\right] = -\eta \mathbf A (\mathbf A^2 \odot {\mathbf S})\mathbf A  [\mathbf R_t, \mathbf 0]  ~+~ \left[\mathbf R_t, \mathbf R'_{t}\right].
\end{split}
\end{equation}

\paragraph{Case 2.} If we stipulate the input features are node embeddings from spectral decomposition of a full-rank adjacency matrix, i.e.  $\bar{\mathbf X} \triangleq \mathop{\arg\min}_{\mathbf B} \|\mathbf A - \mathbf B \mathbf B^\top\|_F^2$. We have 
\begin{equation}
\begin{split}
&~~\mathbf \Theta^{(2)}(\boldsymbol x_i, \boldsymbol x_j; \mathbf A) \\
&= \sum_{{i'} \in \mathcal{N}(i)} \sum_{{j'} \in \mathcal{N}(j)} \mathbf A_{ii'} \mathbf A_{jj'} \mathbf A^3_{i'j'} \mathbb{E}_{\boldsymbol{w} \sim \mathcal{N}(0, 1)}\left[\mathds{1}\left\{\boldsymbol{w}^{\top} [\mathbf A\bar{\mathbf X}]_{i'} \geq 0, \boldsymbol{w}^{\top} [\mathbf A\bar{\mathbf X}]_{j'} \geq 0\right\}\right]\\
&= \sum_{{i'} \in \mathcal{N}(i)} \sum_{{j'} \in \mathcal{N}(j)} \mathbf A_{ii'} \mathbf A_{jj'} \frac{\mathbf A^3_{i'j'}(\pi - \operatorname{arccos}(\frac{[\mathbf A\bar{\mathbf X}]_{i'}^\top[\mathbf A\bar{\mathbf X}]_{j'}}{\|[\mathbf A\bar{\mathbf X}]_{i'}\| \|[\mathbf A\bar{\mathbf X}]_{j'}\|}))}{2\pi}.
\end{split}
\end{equation}
The above equation can be written neatly in the matrix form
\begin{equation}
\mathbf \Theta^{(2)}(\bar{\mathbf X}, \bar{\mathbf X}; \mathbf A) = \mathbf A (\mathbf A^3 \odot \tilde{\mathbf S})\mathbf A, \text{\quad where\quad}\tilde{\mathbf S}_{ij} = {(\pi - \operatorname{arccos}(\frac{[\mathbf A\bar{\mathbf X}]_{i}^\top[\mathbf A\bar{\mathbf X}]_{j}}{\|[\mathbf A\bar{\mathbf X}]_{i}\| \|[\mathbf A\bar{\mathbf X}]_{j}\|}))}/{2\pi}.
\end{equation}
The matrix $\tilde{\mathbf S}$ reweights each entry in $\mathbf A^3$ by the similarity of aggregated node embeddings $[\mathbf A\bar{\mathbf X}]_{i}$ and $[\mathbf A\bar{\mathbf X}]_{j}$. Substituting it to (\ref{eqn_inference}) gives us the training dynamics of overparameterized two-layer GNN
\begin{equation}
\begin{split}
\left[\mathbf R_{t+1}, \mathbf R'_{t+1}\right] =  -\eta \mathbf A (\mathbf A^3 \odot \tilde{\mathbf S})\mathbf A  [\mathbf R_t, \mathbf 0]  ~+~ \left[\mathbf R_t, \mathbf R'_{t}\right].
\end{split}
\end{equation}

\subsection{Proof: Arbitrarily Deep Decoupled GNN} \label{app_decoupled}

Before analyzing the training dynamics of deep GNNs, let us first consider an arbitrarily deep standard fully-connected neural networks, which is denoted as $\operatorname{MLP}(\bar{\mathbf X})$ with weights $\mathbf W$. We also denote its corresponding NTK in overparameterized regime as $\mathbf \Theta^{(\ell)}(\bar{\mathbf X}, \bar{\mathbf X})$. For input features as an identity matrix $\bar{\mathbf X} = \mathbf I_n$, we have the following permutation invariant property of $\mathbf \Theta^{(\ell)}(\bar{\mathbf X}, \bar{\mathbf X})$.

\begin{lemma}[Permutation Invariance of NTK]
For arbitrarily deep and infinitely wide fully-connected neural networks $\operatorname{MLP}(\bar{\mathbf X})$ with ReLU activation and standard NTK parameterization, the corresponding NTK $\mathbf \Theta^{(\ell)}(\bar{\mathbf X}, \bar{\mathbf X})$ for onehot vector inputs is permutation invariant. Namely for arbitrary permutation function $\psi$ on input set $\bar{\mathbf X}$
\begin{equation}
\mathbf \Theta^{(\ell)}(\psi(\bar{\mathbf X}), \psi(\bar{\mathbf X})) = \mathbf \Theta^{(\ell)}(\bar{\mathbf X}, \bar{\mathbf X}).
\end{equation}
\end{lemma}

\begin{proof}
Recall that in the explicit form computation for NTK of fully-connected neural networks~\citep{jacot2018neural,lee2019wide}, the NTK and NNGP for a $(\ell+1)$-layer network is produced from the NTK and NNGP for a $\ell$-layer network. Therefore, to prove the permutation invariance of NTK of an arbitrarily deep fully-connected neural network, we need only to prove the permutation invariance of $\mathbf \Theta^{(1)}$ and $\mathbf \Sigma^{(1)}$, which are defined as the inner product of input features
\begin{equation}
\mathbf \Theta^{(1)}(\bar{\mathbf X}, \bar{\mathbf X}) = \mathbf \Sigma^{(1)}(\bar{\mathbf X}, \bar{\mathbf X}) = \bar{\mathbf X}\bar{\mathbf X}^\top.
\end{equation}
Since arbitrary permutation function $\psi$ on $\bar{\mathbf X}$ can be expressed as
\begin{equation}
\psi(\bar{\mathbf X}) = \prod_{i=0}^c \mathbf T_{i} \bar{\mathbf X}
\end{equation}
for some row-interchanging elementary matrices $\{\mathbf T_i\}_{i=0}^c$. It follows that
\begin{eqnarray}
\mathbf \Theta^{(1)}(\psi(\bar{\mathbf X}), \psi(\bar{\mathbf X})) &=& \prod_{i=0}^c \mathbf T_{i} \bar{\mathbf X} (\prod_{i=0}^c \mathbf T_{i} \bar{\mathbf X})^\top\nonumber\\
&=& \prod_{i=0}^c \mathbf T_{i} \bar{\mathbf X} \bar{\mathbf X}^\top \prod_{i=0}^c \mathbf T_{c-i}^\top\nonumber\\
&=& \mathbf T_0 \cdots \mathbf T_c \mathbf T_c^\top \cdots \mathbf T_0^\top\nonumber\\
&=& \mathbf I_n \nonumber\\
&=& \mathbf \Theta^{(1)}(\bar{\mathbf X}, \bar{\mathbf X}).
\end{eqnarray}
By induction, it follows that $\mathbf \Theta^{(\ell)}(\bar{\mathbf X},\bar{\mathbf X})$ is also permutation invariant.
\end{proof}

Now, let us consider a deep GNN defined as 
\begin{equation}
[\mathbf F, \mathbf F']=\mathbf{A}^{\ell} \operatorname{MLP}(\bar{\mathbf X})
\end{equation}
where $\operatorname{MLP}(\bar{\mathbf X})$ is a $L$-layer infinitely-wide fully-connected neural network with ReLU activation. Then, the corresponding node-level GNTK matrix is defined as
\begin{equation}
\left\langle \frac{\partial \mathbf{A}^{\ell} \operatorname{MLP}(\bar{\mathbf X})}{\partial \mathbf W}, \frac{\partial \mathbf{A}^{\ell} \operatorname{MLP}(\bar{\mathbf X})}{\partial \mathbf W}\right\rangle = \mathbf{A}^{\ell} \mathbf \Theta^{(L)}(\bar{\mathbf X},\bar{\mathbf X})\mathbf{A}^{\ell}.
\end{equation}
Since the NTK matrix $\mathbf \Theta^{(L)}(\bar{\mathbf X},\bar{\mathbf X})$ is permutation invariant with diagonal values being larger than non-diagonal values, we can write it as
\begin{equation}
\mathbf \Theta^{(L)}(\bar{\mathbf X},\bar{\mathbf X}) = c'(\mathbf{I}+c \mathbf{1 1 ^ { \top }})
\end{equation}
for some constants $c'$ and $c$ determined by the depth of the network $L$. Substituting it to (\ref{eqn_inference}) gives us the training dynamics of infinitely-wide and arbitrarily-deep GNN
\begin{equation}
\left[\mathbf R_{t+1}, \mathbf R'_{t+1}\right] = -\eta \mathbf A^{\ell} (\mathbf I + c\mathbf 1 \mathbf 1^\top) \mathbf A^{\ell} [\mathbf R_t, \mathbf 0]  ~+~ \left[\mathbf R_t, \mathbf R'_{t}\right].
\end{equation}

\subsection{Proof of Corollary~\ref{thm_lineargnn}: Linear GNN} \label{app_lineargnn}
For linear GNN defined as $[\mathbf F, \mathbf F']=\mathbf{A}^{\ell} \bar{\mathbf X} \mathbf{W}$. The corresponding node-level GNTK $\mathbf{\Theta}_t^{(1)}$ for linear GNN is naturally constant and can be computed as 
\begin{equation}
\begin{split}
\mathbf{\Theta}_t^{(1)}\left(\bar{\mathbf X}, \bar{\mathbf X}; \mathbf A\right) &=\nabla_{\mathbf{W}} [\mathbf F, \mathbf F']^{\top} \nabla_{\mathbf{W}} [\mathbf F, \mathbf F']\\
&= \mathbf A^\ell \bar{\mathbf X} (\mathbf A^\ell \bar{\mathbf X})^\top.
\end{split}
\end{equation}
\paragraph{Case 1.} When the input $\bar{\mathbf X}$ is defined as an identity matrix, we have 
\begin{equation}
\mathbf A^\ell \bar{\mathbf X} (\mathbf A^\ell \bar{\mathbf X})^\top = \mathbf A^{2\ell}
\end{equation}
which is also a special case of Theorem~\ref{thm_decoupledgnn} where the backbone MLP model is one layer and consequently $c = 0$.
Based on (\ref{eqn_inference}), the training dynamics of linear GNN can be written as
\begin{equation}
\left[\mathbf R_{t+1}, \mathbf R'_{t+1}\right] =  -\eta \mathbf A^{2\ell} [\mathbf R_t, \mathbf 0]  + \left[\mathbf R_t, \mathbf R'_{t}\right]
\end{equation}
which is equivalent to the basic version of RP in (\ref{eqn_rp}) with $K=2\ell$.

\paragraph{Case 2.} When the input $\bar{\mathbf X}$ is obtained from (full-rank) graph spectral decomposition, we have 
\begin{equation}
\mathbf A^\ell \bar{\mathbf X} (\mathbf A^\ell \bar{\mathbf X})^\top = \mathbf A^{\ell} \bar{\mathbf X} \bar{\mathbf X}^\top \mathbf A^{\ell} = \mathbf A^{2\ell+1}.
\end{equation}
In this case, the training dynamics of linear GNN can be written as
\begin{equation}
\left[\mathbf R_{t+1}, \mathbf R'_{t+1}\right] =  -\eta \mathbf A^{2\ell+1}  [\mathbf R_t, \mathbf 0]  + \left[\mathbf R_t, \mathbf R'_{t}\right]
\end{equation}
which is equivalent to the basic version of RP in (\ref{eqn_rp}) with $K=2\ell+1$.
For non-PSD adjacency matrix $\mathbf A$, we also have the approximation $\mathbf A^\ell \bar{\mathbf X}^\top \bar{\mathbf X} \mathbf A^\ell \approx \mathbf A^{2\ell+1}$. However, in this case, the basic RP algorithm in (\ref{eqn_rp}) can not be implemented by the conventional (deep) learning framework (cf. Theorem~\ref{thm_converge}).

\subsection{Derivation of \eqref{eq_genbound}: Generalization Bound}\label{app_gen_bound} 

Based on the Rademacher complexity-based generalization bound for kernel regression in~\citep{bartlett2002rademacher}, for training data $\{(\boldsymbol x_i, y_i)\}_{i=1}^{n_l}$ drawn i.i.d. from an underlying distribution $\mathcal P$, arbitrary loss function $l:\mathbb R \times \mathbb R\rightarrow [0,1]$ that is 1-Lipschitz in the first argument such that $l(y, y) = 0$, we have that with probability at least $1-\delta$, the population risk of kernel regression with respect to the limit NTK $\mathbf \Theta$ has upper bound (see proof in \citet{bartlett2002rademacher,arora2019fine,du2019graph})
\begin{equation}
\mathbb{E}_{(\boldsymbol x,y) \sim \mathcal{P}}\left[l\left(f(\boldsymbol x), y\right)\right]=O\left(\frac{\sqrt{\mathbf Y^\top \mathbf \Theta^{-1} \mathbf Y \cdot\operatorname{Tr}(\mathbf \Theta)}}{n_l}+\sqrt{\frac{\log (1 / \delta)}{n_l}}\right).
\end{equation}
We now assume an extreme case where the limit NTK matrix is determined by the graph, i.e. $\mathbf \Theta = (\mathbf I - \alpha \mathbf A)^{-1} = \lim _{k \rightarrow \infty} \sum_{i=0}^{k}(\alpha \mathbf A)^i$, which is equivalent to the propagation matrix adopted by the converged LP algorithm,\footnote{With slight abuse of notation, we denote by $\mathbf A$ the adjacency matrix for the training set, and $\mathbf \Theta^*$ the optimal
kernel matrix for the training set.} and is a valid kernel matrix by noting that the spectral radius $\rho(\mathbf A)\leq 1$ for normalized adjacency matrix $\mathbf A$ and $\alpha < 1$. For $\mathbf Y^\top \mathbf \Theta^{-1} \mathbf Y$, we have
\begin{eqnarray}
\mathbf Y^\top \mathbf \Theta^{-1} \mathbf Y &=& \mathbf Y^\top \mathbf Y - \alpha \mathbf Y^\top \mathbf A \mathbf Y \nonumber\\
&=& n_l - \alpha \langle\mathbf Y\mathbf Y^\top, \mathbf A\rangle_F \nonumber\\
&=& n_l - c A(\mathbf \Theta^*, \mathbf A)
\end{eqnarray}
where $c = \alpha \|\mathbf \Theta^*\|_F \|\mathbf A\|_F$ is a constant. For $\operatorname{Tr}(\mathbf \Theta)$, we have
\begin{equation}
\operatorname{Tr}(\mathbf \Theta) = \sum_{i=1}^{n_l} \frac{1}{1-\alpha\sigma_i(\mathbf A)} \leq \frac{n_l}{1-\alpha\sigma_{max}(\mathbf A)} = O(n_l),
\end{equation}
where $\sigma_{max}(\mathbf A)\leq 1$ is the maximal eigenvalue of $\mathbf A$.
It follow that the population risk has upper bound
\begin{equation}
\mathbb{E}_{(\boldsymbol x,y) \sim \mathcal{P}}\left[l\left(f(\boldsymbol x), y\right)\right]=O\left(\sqrt{\frac{n_l - c A(\mathbf \Theta^*, \mathbf A)}{n_l}}+\sqrt{\frac{\log (1 / \delta)}{n_l}}\right).
\end{equation}


\subsection{Proof of Theorem~\ref{thm_gnnoptimal}: Bayesian Optimality of GNN}\label{app_gnnoptimal}
\paragraph{Setup.} In this proof, we consider the following setup. Suppose all possible samples are given by inputs $\mathbf X = \{\boldsymbol x_i\}_{i=1}^N$ and labels $\mathbf Y = \{y(\boldsymbol x_i)\}_{i=1}^N$ that are randomly generated from an unknown distribution $\mathcal P$, where $N$ could be arbitrarily large and both $\mathbf X$ and $\mathbf Y$ are assumed to be unobserved. For binary classification, we also have $y\in \{-1,1\}$ with balanced label distribution, i.e. $\mathbb{E}[y] = 0$. Given that for an arbitrary pair of instances $\boldsymbol x$ and $\boldsymbol x'$, the probability that they share the same label is defined by a large matrix $\mathbf A \in \mathbb R^{N\times N}$:
\begin{equation}\label{eqn_prob}
P\left(y(\boldsymbol x) = y(\boldsymbol x')\right) = \mathbf A_{\boldsymbol x\boldsymbol x'},
\end{equation}
where $\mathbf A_{\boldsymbol x\boldsymbol x'}$ is an element in $\mathbf A$ that corresponds to sample pair $\boldsymbol x$ and $\boldsymbol x'$. Our task is to find the optimal kernel function whose corresponding kernel regression predictive function $f_{ker}(\cdot)$ minimizes the population risk
\begin{equation}
\mathbb{E}_{(\boldsymbol x,y) \sim \mathcal{P}}\left[\left(f_{ker}(\boldsymbol x) - y(\boldsymbol x)\right)^2\right].
\end{equation}

\paragraph{Proof.} Next, we give solution to the optimal kernel function that minimizes the population risk in the above setting. For data randomly generated from $\mathcal P$ and conditioned on (\ref{eqn_prob}), the random labels have covariance matrix $\mathbf \Sigma(\mathbf X, \mathbf X) \in \mathbb R^{N\times N}$ which satisfies
\begin{equation}
\begin{split}
    \mathbf \Sigma(\boldsymbol x, \boldsymbol x') &= \mathbb{E}_{(\boldsymbol x,y) \sim \mathcal{P}}\left[(y(\boldsymbol x) - \mathbb{E}[y])(y(\boldsymbol x')-\mathbb{E}[y])\right]\\
    &= \mathbb{E}_{(\boldsymbol x,y) \sim \mathcal{P}}\left[y(\boldsymbol x) y(\boldsymbol x')\right] \\
    &= 2\mathbf A_{\boldsymbol x\boldsymbol x'} - 1.
\end{split}
\end{equation}
Let us consider kernel regression w.r.t. a kernel $\kappa$, whose corresponding predictive function $f_{ker}(\boldsymbol x)$ on all possible unobserved samples $\mathbf X$ and $\mathbf Y$ is given by
\begin{equation}
f_{ker}(\boldsymbol x) = \kappa(\boldsymbol x, \mathbf X) \mathbf K(\mathbf X, \mathbf X)^{-1} \mathbf Y = \mathbf M^\top_{\boldsymbol x} \mathbf Y.
\end{equation}
where $\mathbf M_{\boldsymbol x} \in \mathbb R^{N}$ is a vector associated with sample $\boldsymbol x$ defined as $\mathbf M_{\boldsymbol x}^\top = \kappa(\boldsymbol x, \mathbf X) \mathbf K(\mathbf X, \mathbf X)^{-1}$.
Note that we do not differentiate training or testing samples here since all inputs and labels are assumed to be unseen.
We aim to search the optimal kernel function that minimizes the population risk, which can be achieved if the risk for each sample $\boldsymbol x$ is minimized. Specifically, for an arbitrary sample $\boldsymbol x$, its risk is
\begin{eqnarray}
\mathbb{E}_{y}\left[\left(f_{ker}(\boldsymbol x) - y(\boldsymbol x)\right)^2\right] &=& \mathbb{E}_{y}\left[\left(\mathbf M^\top_{\boldsymbol x} \mathbf Y - y(\boldsymbol x)\right)^2\right]\nonumber\\
&=& \mathbf M^\top_{\boldsymbol x} \mathbf \Sigma(\mathbf X, \mathbf X) \mathbf M_{\boldsymbol x} -2 \mathbf M_{\boldsymbol x}^\top \mathbf \Sigma(\mathbf X, \boldsymbol x) + \mathbf \Sigma(\boldsymbol x, \boldsymbol x).
\end{eqnarray}
To find the optimal kernel that minimizes the risk, we differentiate it w.r.t. $\mathbf M_{\boldsymbol x}$, 
\begin{equation}
\nabla_{\mathbf M_{\boldsymbol x}} \mathbb{E}_{(\boldsymbol x,y) \sim \mathcal{P}}\left[\left(f_{ker}(\boldsymbol x) - y(\boldsymbol x)\right)^2\right] =2\mathbf \Sigma(\mathbf X, \mathbf X) \mathbf M_{\boldsymbol x} - 2 \mathbf \Sigma(\mathbf X, \boldsymbol x)
\end{equation}
which gives us the minimizer
\begin{equation}
\mathbf M_{\boldsymbol x}^\top = \mathbf \Sigma(\boldsymbol x, \mathbf X) \mathbf \Sigma(\mathbf X, \mathbf X)^{-1}.
\end{equation}
And since $\mathbf M_{\boldsymbol x}^\top = \kappa(\boldsymbol x, \mathbf X) \mathbf K(\mathbf X, \mathbf X)^{-1}$, the population risk is minimized on each single data point for optimal kernel $\kappa^*$ whose kernel matrix is
\begin{equation}
\mathbf K^*(\mathbf X, \mathbf X) = 2\mathbf A-1.
\end{equation}

\begin{remark}
The above analysis also generalize to the setting when $\mathbf A$ is defined over a subset of all possible samples, in which case the optimal kernel matrix is still $2\mathbf A-1$ but it is no longer possible to minimize population risk outside the coverage of $\mathbf A$.
\end{remark}

\section{Additional Discussions} \label{app_additional_discussion}

\subsection{Different Settings} \label{app_induc_trans}

Transductive (semi-supervised) and inductive (supervised) learning are two types of common settings in node-level classification tasks. The former incorporates unlabeled nodes (testing samples) in the training process while the latter only has access to unlabeled nodes for inference. Recall that the residual propagation process of general parameterized models with arbitrary unseen testing samples can be written as
\begin{equation}
\left[\mathbf R_{t+1}, {\mathbf R'_{t+1}}\right] = -\eta~ \left[\begin{array}{cc}
\underbrace{\mathbf \Theta_t(\mathbf X, \mathbf X)}_{{n_l\times n_l}} & \mathbf \Theta_t(\mathbf X, \mathbf X') \\
\underbrace{{\mathbf \Theta_t(\mathbf X', \mathbf X)}}_{{(n-n_l)\times n_l}} & \mathbf \Theta_t(\mathbf X', \mathbf X')
\end{array}\right] \cdot \underbrace{[{\mathbf R_t}, \mathbf 0]}_{{n}}  ~~+~~ \underbrace{\left[\mathbf R_t, {\mathbf R'_{t}}\right]}_{{n}}.
\end{equation}

For node-level tasks, the difference between transductive and inductive settings boils down to different feature map of the kernel function for training and testing sets. To be specific:

\paragraph{Transductive learning.} For transductive learning or semi-supervised learning, the residual propagation can be written as
\begin{equation}
\left[\mathbf R_{t+1}, {\mathbf R'_{t+1}}\right] = -\eta~ \left[\begin{array}{cc}
\underbrace{\blue{\mathbf \Theta_t(\mathbf X, \mathbf X; \mathbf A)}}_{n_l\times n_l} & \mathbf \Theta_t(\mathbf X, \mathbf X'; \mathbf A) \\
\underbrace{\red{\mathbf \Theta_t(\mathbf X', \mathbf X; \mathbf A)}}_{(n-n_l)\times n_l} & \mathbf \Theta_t(\mathbf X', \mathbf X'; \mathbf A)
\end{array}\right] \cdot \underbrace{[{\mathbf R_t}, \mathbf 0]}_{n}  ~~+~~ \underbrace{\left[\mathbf R_t, {\mathbf R'_{t}}\right]}_{n}.
\end{equation}
where $\mathbf \Theta_t(\mathbf X, \mathbf X; \mathbf A)$ is the node-level GNTK defined in Sec.~\ref{sec_gnndynamics}. This equation is equivalent to the one presented in the main text, i.e. (\ref{eqn_inference}).

\paragraph{Inductive learning.} Let us denote $\mathbf A_{\mathbf X\mathbf X}$ the submatrix of $\mathbf A$ corresponding to the training set. For inductive learning or supervised learning, the residual propagation can be written as
\begin{equation}
\left[\mathbf R_{t+1}, {\mathbf R'_{t+1}}\right] = -\eta~ \left[\begin{array}{cc}
\underbrace{\blue{\mathbf \Theta_t(\mathbf X, \mathbf X; \mathbf A_{\mathbf X\mathbf X})}}_{n_l\times n_l} & \mathbf \Theta^{(ind)}_t(\mathbf X, \mathbf X'; \mathbf A) \\
\underbrace{\red{\mathbf \Theta^{(ind)}_t(\mathbf X', \mathbf X; \mathbf A)}}_{(n-n_l)\times n_l} & \mathbf \Theta^{(ind)}_t(\mathbf X', \mathbf X'; \mathbf A)
\end{array}\right] \cdot \underbrace{[{\mathbf R_t}, \mathbf 0]}_{n}  ~~+~~ \underbrace{\left[\mathbf R_t, {\mathbf R'_{t}}\right]}_{n},
\end{equation}
Let us denote predictions (for the training set) given by the GNN model with matrix $\mathbf A_{\mathbf X\mathbf X}$ as $\mathbf F_{ind}$, and the predictions (for the testing set) given by the GNN model with matrix $\mathbf A$ as $\mathbf F'_{trans}$, we have
\begin{equation}
\begin{split}
\mathbf \Theta_t(\mathbf X, \mathbf X; \mathbf A_{\mathbf X\mathbf X}) &= \nabla_{\mathbf W} \mathbf F_{ind}^\top \nabla_{\mathbf W} \mathbf F_{ind},\\
\mathbf \Theta^{(ind)}_t(\mathbf X', \mathbf X; \mathbf A) &= \nabla_{\mathbf W} {\mathbf F'}_{trans}^\top \nabla_{\mathbf W} \mathbf F_{ind}.
\end{split}
\end{equation}
In this case, the $n\times n$ matrix is still a valid kernel matrix with graph implicit bias (cf. Appendix~\ref{app_gntk}), and thus our insights from the transductive setting still holds.

\paragraph{Training without graph.} Another interesting setting studied recently~\citep{yang2022graph} is using vanilla MLP for training and then adopts the GNN architecture in inference, which leads to the so-called PMLP that can accelerate training without deteriorating the generalization performance. In this setting, the residual propagation can be written as 
\begin{equation}
\left[\mathbf R_{t+1}, {\mathbf R'_{t+1}}\right] =  -\eta~ \left[\begin{array}{cc}
\underbrace{\blue{\mathbf \Theta_t(\mathbf X, \mathbf X)}}_{n_l\times n_l} & \mathbf \Theta^{(pmlp)}_t(\mathbf X, \mathbf X'; \mathbf A) \\
\underbrace{\red{\mathbf \Theta^{(pmlp)}_t(\mathbf X', \mathbf X; \mathbf A)}}_{(n-n_l)\times n_l} & \mathbf \Theta^{(pmlp)}_t(\mathbf X', \mathbf X'; \mathbf A)
\end{array}\right] \cdot \underbrace{[{\mathbf R_t}, \mathbf 0]}_{n}  ~~+~~ \underbrace{\left[\mathbf R_t, {\mathbf R'_{t}}\right]}_{n},
\end{equation}
where $\mathbf \Theta_t(\mathbf X, \mathbf X)$ is equivalent to the NTK matrix of fully-connected neural networks. This is equivalent to the inductive setting where the graph adjacency matrix used in training is an identity matrix: let us denote predictions (for the training set) given by MLP as $\mathbf F_{mlp}$, we have
\begin{equation}
\begin{split}
\mathbf \Theta_t(\mathbf X, \mathbf X) &= \nabla_{\mathbf W} \mathbf F_{mlp}^\top \nabla_{\mathbf W} \mathbf F_{mlp},\\
\mathbf \Theta^{(ind)}_t(\mathbf X', \mathbf X; \mathbf A) &= \nabla_{\mathbf W} {\mathbf F'}_{trans}^\top \nabla_{\mathbf W} \mathbf F_{mlp}.
\end{split}
\end{equation}
The success of PMLP is consistent with the insight that the residual flow between training and testing sets with graph implicit bias is important to explain the generalization of GNNs.

\subsection{Other Loss Functions} \label{app_loss}
It is worth noting that residual propagation process is not exclusive to the squared loss. While our analysis focuses on the squared loss, similar residual propagation schemes can be obtained with other loss functions that will change the original linear residual propagation process to non-linear. Therefore, the insight still holds for other loss functions.

\paragraph{Mean squared error.} For completeness, let us rewrite the residual propagation process for general parameterized model (e.g. linear model, fully-connected neural network, GNN) with squared loss here, with derivation given in Appendix~\ref{proof_rp_ntk}
\begin{equation}
\left[\mathbf R_{t+1}, {\mathbf R'_{t+1}}\right] =  -\eta~ \mathbf \Theta_t(\bar{\mathbf X}, \bar{\mathbf X}) [{\mathbf R_t}, \mathbf 0] ~+~ \left[\mathbf R_t, \mathbf R'_{t}\right],
\end{equation}
which is a linear propagation process, where $\mathbf \Theta_t(\bar{\mathbf X}, \bar{\mathbf X}) = \nabla_{\mathbf W} [\mathbf F_t, \mathbf F'_t]^\top ~ \nabla_{\mathbf W} [\mathbf F_t, \mathbf F'_t] \in \mathbb R^{n\times n}$.

\paragraph{Mean absolute error.}
If the loss function is a MAE (a.k.a. $L_1$) loss $\mathcal L = \|\mathbf F_t - \mathbf Y\|_1$, then the residual propagation process can be revised as
\begin{equation}
\left[\mathbf R_{t+1}, {\mathbf R'_{t+1}}\right] =  -\eta~ \mathbf \Theta_t(\bar{\mathbf X}, \bar{\mathbf X}) [\operatorname{sgn}(\mathbf R_t), \mathbf 0] ~+~ \left[\mathbf R_t, \mathbf R'_{t}\right],
\end{equation}
where $\operatorname{sgn}(x)$ is the sign function that output $1$ is $x > 0$ otherwise $-1$, and could be thought of as coarsening the residual information to a binary value. 

\paragraph{Cross-entropy.}
For the CE loss commonly used for classification tasks, we define the residual as $\mathbf R_t = \mathbf Y - \sigma \mathbf F_t$, where $\sigma$ here denotes sigmoid activation and applies on every element in vector $\mathbf F_t$. The residual propagation process can be revised as 
\begin{equation}
\left[\mathbf R_{t+1}, {\mathbf R'_{t+1}}\right] =  -\eta~ \mathbf \Theta_t(\bar{\mathbf X}, \bar{\mathbf X}) [c(\mathbf R_t), \mathbf 0] ~+~ \left[\mathbf R_t, \mathbf R'_{t}\right],
\end{equation}
where the function $c(\cdot)$ denotes multiplying a scaling factor $1/(1-\sigma f_t(\boldsymbol x_i))\sigma f_t(\boldsymbol x_i)$ to each residual for re-weighting. This scaling factor up-weights smaller residuals whose corresponding predictions are more confident, and down-weights larger residuals whose predictions are less confident.

\subsection{Multi-Dimensional Output} \label{app_multioutput}
In this section, we discuss the modification of analysis for training dynamics of GNNs when outputs $f(\boldsymbol x) \in \mathbb R^c$ is multi-dimensional.
For the multi-dimensional output case where $f(\boldsymbol x) \in \mathbb R^{c}$ and $c$ is the dimension of outputs, let $\mathbf F_t \in \mathbb R^{n_l c\times 1}$ and $\mathbf R_t \in \mathbb R^{n_l c\times 1}$ be model prediction and residuals in vector forms. Then, the residual dynamics for the training set (which can straight-forwardly incorporate testing samples similar to the derivation in Appendix~\ref{proof_rp_ntk}) is revised as 
\begin{align}
\frac{\partial \mathbf R_t}{\partial t} &= \frac{\partial \mathbf R_t}{\partial \mathbf W_t} \cdot \frac{\partial \mathbf W_t}{\partial t} && \text{(Chain rule)} \\
&= - \eta\cdot \frac{\partial \mathbf R_t}{\partial \mathbf W_t} \cdot \nabla_{\mathbf W} {\mathcal L} && \text{(GD training)}\\
&= -\eta\cdot \frac{\partial \mathbf R_t}{\partial \mathbf W_t} \cdot \nabla_{\mathbf W} \mathbf F_t\cdot \nabla_{f} \mathcal L && \text{(Chain rule)} \\
&= \eta\cdot \underbrace{\nabla^\top_{\mathcal W} \mathbf F_t}_{{n_l c\times |\mathbf W|}} \cdot \underbrace{\nabla_{\mathbf W} \mathbf F_t}_{{{|\mathbf W|\times n_l c}}}\cdot \underbrace{\nabla_{f} \mathcal L}_{{n_l c\times 1}} && \text{(Change of notation)} \\
&= \eta\cdot \underbrace{\mathbf \Theta_t(\mathbf X, \mathbf X)}_{{n_l c\times n_l c}} \cdot \underbrace{\nabla_{f} \mathcal L}_{{n_l c\times 1}} \\
&= -\eta \cdot \underbrace{\mathbf \Theta_t(\mathbf X, \mathbf X)}_{{n_l c\times n_l c}} \cdot \underbrace{\mathbf R_t}_{{n_l c\times 1}} && \text{(Squared loss)}.
\end{align}

Compared with residual propagation process for scalar output, the residual propagation process for multi-dimensional output additionally incorporate flow of residual across different dimensions. However, for infinite neural networks, the NTK $\mathbf \Theta (\mathbf X, \mathbf X)$ for multi-dimensional output~\citep{jacot2018neural} can be written as
\begin{equation}
\underbrace{\mathbf \Theta (\mathbf X, \mathbf X)}_{{n_l c\times n_l c}} = \underbrace{\bar{\mathbf \Theta} (\mathbf X, \mathbf X)}_{{n_l\times n_l}} ~~\otimes \underbrace{\mathbf I_{c}}_{{c \times c}}.
\end{equation}

\begin{figure}[t]
\centering
\begin{minipage}[t]{\linewidth}
\hspace{5pt}
\begin{minipage}[t]{0.33\linewidth}
\centering
\includegraphics[width=0.99\textwidth]{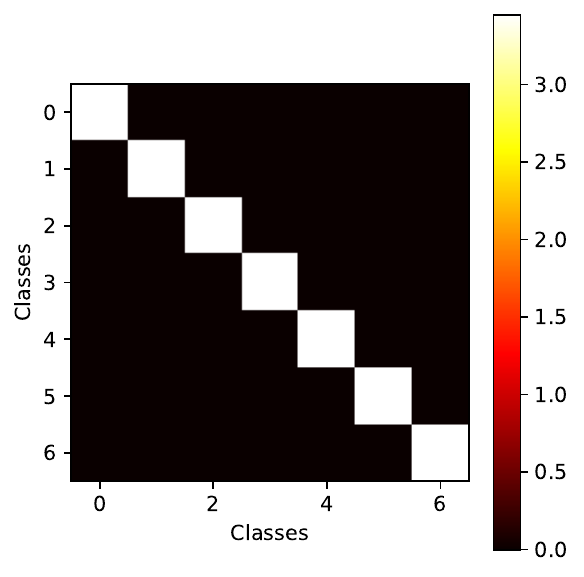}
\end{minipage}%
\begin{minipage}[t]{0.33\linewidth}
\centering
\includegraphics[width=0.99\textwidth]{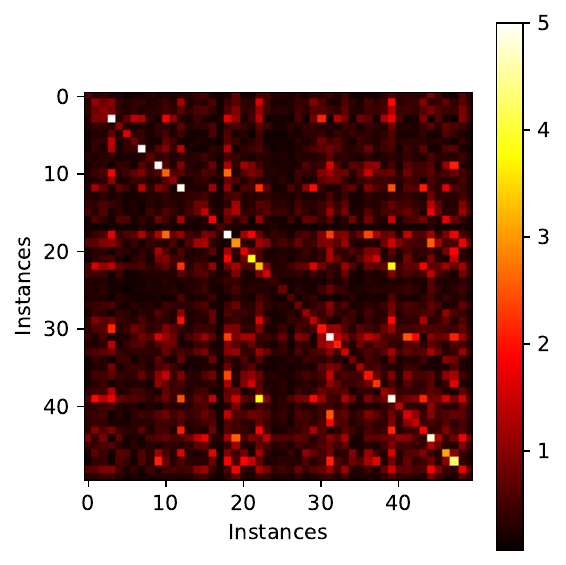}
\end{minipage}%
\begin{minipage}[t]{0.33\linewidth}
\centering
\includegraphics[width=0.99\textwidth]{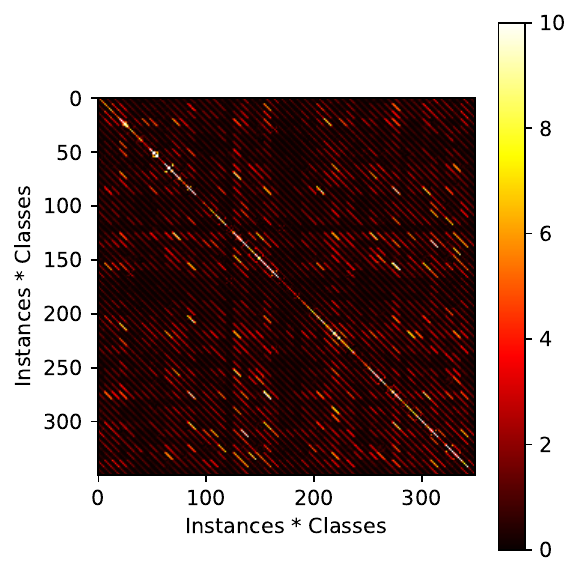}
\end{minipage}%
\end{minipage}
\caption{Visualization of NTK of well-trained GCN on a node classification benchmark (Cora). $50$ nodes are randomly selected for clearity. From left to right are $c\times c$, $n\times n$, $nc\times nc$ NTK matrices, where the former two matrices are obtained by averaging the $nc\times nc$ NTK matrix at dimension $n$ and $c$ respectively. The diagonal patterns in the first and last matrix verifies that our analysis for finitely-wide GNNs in binary classification also applies to multi-class classification setting.}
\label{fig_ntkvis}
\end{figure}

Namely, there is no cross-dimension flow in residual propagation and our analysis for scalar output can be trivially adapted to the multi-dimensional output case. For finite neural networks, while cross-dimension residual propagation is theoretically inevitable, one can still resort the the above decomposition as approximation. For example, we can define pseudo NTK as 
\begin{equation}
\bar{\mathbf \Theta}_t(\mathbf X,\mathbf X)= \left(\nabla_{\mathbf W} \sum_{h=1}^c [\mathbf F_t]_h\right)^\top  \left(\nabla_{\mathbf W} \sum_{h=1}^c [\mathbf F_t]_h\right),
\end{equation}
and recent works~\citep{mohamadi2022fast} has proved that such a kernel function can be used to approximate the original $n_l c\times n_l c$ kernel matrix:
\begin{equation}
\frac{\left\|\bar{\mathbf \Theta}_t\left(\boldsymbol x_i, \boldsymbol x_j\right) \otimes \mathbf I_c - \mathbf \Theta_t\left(\boldsymbol x_i, \boldsymbol x_j\right)\right\|_F}{\left\|\mathbf \Theta_t\left(\boldsymbol x_i, \boldsymbol x_j\right)\right\|_F} \in \tilde{\mathcal{O}}\left(n_l^{-\frac{1}{2}}\right).
\end{equation}
To empirically verify this for finitely-wide GNNs, we visualize the NTK matrix for real-world GNNs in Cora dataset in Fig.~\ref{fig_ntkvis}.
Therefore, our insights for both infinitely-wide GNNs and finitely-wide GNNs still hold in multi-dimensional case, which are also verified by our experiments.

\subsection{Practical Implications for Heterophily}\label{app_heterophily}

The analysis in works can be used to provide high-level design principles for handling heterophily. For instances:

Increasing $A\left(\mathbf{A}, \mathbf \Theta^*\right)$ by designing propagation matrices that deviate from the original adjacency matrix, such that the resulting GNN NTK could better align with the optimal kernel matrix, can lead to better generalization performance. Moreover, this perspective helps to substantiate certain existing heuristics for handling heterophily, such as assigning negative weights to edges; 

Decreasing $A\left(\mathbf \Theta_t, \mathbf \Theta^*\right)$ by designing non-standard GNN architectures, i.e., that deviate from our defintion in Section 2, can reduce the degradation of NTK-target alignment caused by heterophily; in doing so, the negative effects of heterophilic graphs on generalization can also be mitigated. Incidently, this also helps to quantify regimes where non-GNN architectures such as MLPs can sometimes perform better.

We note also that even the extremely simple RP algorithm that naturally emerges from our analysis works well in practice handling heterophily. See Appendix~\ref{app_addexp} for empirical examples.

\section{Implementation Details} \label{app_detail}

\subsection{Experiments in Section~\ref{sec_rp_theorem}} \label{app_detail1}

\texttt{Arxiv}, \texttt{Proteins} and \texttt{Products}~\citep{hu2020open} are three relatively large datasets containing $169343$, $132534$ and $2449029$ nodes respectively. 

\begin{itemize}
    \item The \texttt{Arxiv} dataset represents the citation network between all computer science arxiv papers. Each node is associated with a feature vector representing the averaged embeddings of words in the title and abstract of that paper and the task is to predict the subject areas. 
    \item For \texttt{Proteins} dataset, nodes represent proteins and edges represent biologically significant associations between proteins, categorized by their types. The task is to predict the presence or absence of 112 protein functions as a multi-label binary classification problem. 
    \item The \texttt{Products} dataset is an Amazon product co-purchasing network where nodes are products in Amazon and links represent two products are purchased together. The feature for each node is dimensionality-reduced bag-of-words for the product descriptions and the task is predict the category of a product. 
\end{itemize}

We follow the original splitting of \cite{hu2020open} for evaluation. The statistics of these datasets are shown in Table~\ref{tbl:dataset}.

\begin{table*}[t]
    \centering
    \caption{Statistics of $17$ datasets.}
    \label{tbl:dataset}
    \resizebox{\textwidth}{!}{
    \begin{tabular}{lllllccccc}
    \toprule
    Type & Dataset &  \# Nodes ($n$) & \# Edges ($e$) & \# Features ($d$) & \# Class ($c$) & \# Tasks  & Split\\ 
    \midrule
    \multirow{3}{*}{OGB} & \texttt{Arxiv}  & 169,343 & 1,166,243  & 128  & 40 & 1 & Public\\ 
    &\texttt{Proteins}  & 132,534 & 39,561,252  & 8  & 2 & 112 & Public\\ 
    &\texttt{Products} &  2,449,029 & 61,859,140 & 100 & 47 & 1 & Public\\ \midrule
    \multirow{7}{*}{{Homophilic}} &\texttt{Cora}  & 2,708 & 10,556 & 1,433  & 7 & 1 & Public\\ 
    &\texttt{Citeseer}  & 3,327 & 9,104 & 3,703  & 6 & 1 & Public\\ 
    &\texttt{Pubmed}  & 19,717 & 88,648 & 500  & 3 & 1 & Public\\ 
    &\texttt{Computers}  & 13,752 & 491,722 & 767  & 10 & 1 & 80\%/10\%/10\%\\ 
    &\texttt{Photo}  & 7,650 & 238,162 & 745 & 8 & 1 & 80\%/10\%/10\%\\ 
    &\texttt{CS}  & 18,333 & 163,788 & 6,805  & 15 & 1 & 80\%/10\%/10\%\\ 
    &\texttt{Physics}  & 34,493 & 495,924 & 8,415  & 5 & 1 & 80\%/10\%/10\%\\ \midrule
    \multirow{7}{*}{{Heterophilic}}&\texttt{roman-empire}  & 22662 & 32927  & 300  & 18 & 1 & Public\\
    &\texttt{amazon-ratings}  & 24,492 & 93,050  & 300  & 5 & 1 & Public\\
    &\texttt{minesweeper}  & 10,000 & 39,402  & 7  & 2 & 1 & Public\\
    &\texttt{tolokers}  & 11,758 & 519,000  & 10  & 2 & 1 & Public\\
    &\texttt{questions}  & 48,921 & 153,540  & 301  & 2 & 1 & Public\\
    &\texttt{Texas}  & 183 & 325  & 1,703  & 5 & 1 & Public\\ \midrule
    \multirow{1}{*}{{Synthetic}} &\texttt{Synthetic} &  2,000 & 8,023 $\sim$ 8,028  & 100 & 5 & 1 & 100/500/1000\\ \bottomrule
    \end{tabular}
    }
\end{table*}

We compare GRP with several classic methods for learning on graphs: standard \emph{MLP}, Label Propagation (\emph{LP})~\citep{zhu2003semi}, \emph{LinearGNN} (SGC)~\citep{wu2019simplifying}, \emph{GNN (GCN)}~\citep{kipf2016semi}. Except the results of linear GNN which are from our reproduction, the results of other baselines align with the results reported in the OGB leaderboard, where the detailed implementation and hyperparameter setup can be found.~\footnote{\url{https://ogb.stanford.edu/docs/leader_nodeprop}} The LP algorithm reported in Tab.~\ref{tbl_bench} follows the standard implementation that is ran until convergence, while the LP algorithm reported in Fig.~\ref{fig_rpcurve} does not run until convergence in order to align with the proposed RP algorithm. 

For hyperparamter search of RP, we adopt grid search for the RP algorithm with the step size $\eta$ from $\{0.01, 0.02, 0.05, 0.1, 0.2, 0.5, 1\}$, the power $K$ ranging from $1$ to $10$. For \texttt{Arxiv}, \texttt{Proteins} and \texttt{Products}, $K$ is chosen as $7$, $1$, $8$ respectively. Since both LP and GRP are deterministic algorithms, their standard deviations are $0$. All experiments are conducted on Quadro RTX 8000 with 48GB memory.

\subsection{Experiments in Section~\ref{sec_exp4theory}}   \label{app_detail2}
We conduct experiments on real-world benchmark datasets \texttt{Cora} and \texttt{Texas}, and a synthetic dataset generated by the stochastic block model. For the synthetic dataset, we set number of blocks as $5$ (i.e. number of classes) with each block having $400$ nodes. Each node is associated with a $100$-dimensional informative input feature vector. For the homophilic version of the dataset, the $5\times 5$ edge probability matrix is defined as $\mathbf P = 0.01 \cdot\mathbf I_5$, i.e. nodes in the same block are connected with probability $0.01$, and we gradually change this matrix in the generation process until there are only heterophilic edges left in the dataset, i.e. $\mathbf P = 0.0025 \cdot \mathbf 1 \mathbf 1^\top - 0.0025 \cdot \mathbf I_5$. The statistics of these datasets are shown in  Table~\ref{tbl:dataset}. Note that we do not consider the large-scale datasets as in Section~\ref{sec_rp_theorem}, since the computing NTK matrix is extremely costly in memory, especially for GNNs in node-level tasks where the output for an instance is also related to input features of other instances and mini-batch partitioning can not be directly adopted.

For the model, we choose a two-layer GCN with width $16$, bias term and ReLU activation for all datasets. In order to control the variable, when comparing NTK of the model using different graph structures, we fix the weights in the model, which is achieved by training on a fixed graph structure but evaluate the NTK matrix (and test performance) using different graph structure. For real-world datasets, we use the original graph for training which is equivalent to the standard training, and for synthetic dataset, we use an identity matrix for training in which case the model is equivalent to the recently proposed PMLP~\citep{yang2022graph} model that has shown to as effective as GNNs in the supervised learning setting. The optimization algorithm is gradient descent with learning rates $1e-2$, $3e-4$, $5e-5$ respectively for \texttt{Cora}, \texttt{Texas}, \texttt{Synthetic} respectively, momentum $0.9$ and weight decay $5e-4$. The loss function is the standard cross-entropy loss for multi-class classification.

Since the NTK matrix for multi-dimensional output is a $nc\times nc$ matrix, where $c$ is the output dimension, we follow prior work (see \citet{mohamadi2022fast} and references therein) and compute the $n\times n$ NTK matrix by averaging over the dimension $c$. The graph adjacency matrix we consider here is defined as $\mathbf A^{4}$ in order to align with our theoretical result for two-layer infinitely-wide GNN in Theorem~\ref{thm_twolayergnn}. We also normalize these matrices before computing the alignment following~\citep{cortes2012algorithms,baratin2021implicit} as a standard way of preprocess.

\subsection{Algorithm Description} \label{app_detail3}
Recall the basic version of RP can be described by the following iterative forward propagation process:
\begin{equation} 
\left[\mathbf R_{t+1}, \mathbf R'_{t+1}\right] =  -\eta \mathbf A^{K}  [\mathbf R_t, \mathbf 0]  + \left[\mathbf R_t, \mathbf R'_{t}\right].
\end{equation}
The intial conditions are
\begin{equation}
\left[\mathbf R_0, \mathbf R'_{0}\right] = \left[\mathbf Y - \mathbf F_0, \mathbf 0 - \mathbf F'_0\right],
\end{equation}
where $\mathbf F_0 = \mathbf 0$ and $\mathbf F'_0 = \mathbf 0$.
For real-world applications where the outputs are usually multi-dimensional, say, multi-class classification with $c$ classes, we define: $\mathbf Y = \{\boldsymbol y_i\}_{i=1}^{n_l} \in \mathbb R^{n_l\times c}$ where $\boldsymbol y_i$ is a onehot vector, $\mathbf A \in \mathbb R^{n\times n}$ is the same standard normalized adjacency matrix in GCN~\citep{kipf2016semi}. The pseudo code for the basic version of RP are shown in the following algorithm.

\begin{algorithm}[H]
    \caption{Basic version of residual propagation.}
    \KwIn{Raw graph adjacency matrix ${\mathbf A}$, ground-truth labels for training samples $\mathbf Y$, step size $\eta$, power $K$.}
     Compute the normalized graph adjacency matrix by $\mathbf A \gets {\mathbf D}^{-\frac{1}{2}} ({\mathbf A + \mathbf I}) {\mathbf D}^{-\frac{1}{2}}$ \\
     Initialize $\mathbf R_0 \gets \mathbf Y \in \mathbb R^{n_l\times c}$\\
     Initialize $\mathbf R'_0 \gets \mathbf 0 \in \mathbb R^{(n-n_l)\times c}$\\
    \While{Validation performance increases}{
        Label/residual propagation on the graph:\\
        $\tilde{\mathbf R}_{t-1} \gets [\mathbf R_{t-1}, \mathbf 0] \in \mathbb R^{n\times c}$\\
        \For{$i \gets 1$ to $K$}{
            $\tilde{\mathbf R}_{t-1} \gets \mathbf A ~\tilde{\mathbf R}_{t-1}$\\
        }
        Update residuals $[\mathbf R_t, \mathbf R'_t] \gets [\mathbf R_{t-1}, \mathbf R'_{t-1}] - \eta \tilde{\mathbf R}_{t-1}$
    }
    Output prediction for testing samples ${\mathbf F'} =  - \mathbf R'$
\end{algorithm}

\section{Additional Experiments}\label{app_addexp}
\subsection{Incorporating Input Features} \label{app_addexp1}
For smaller datasets where node features are often more useful (i.e. the dimension of node features is closer to the size of dataset, for example in \texttt{Citeseer}, the node feature dimension is even larger than the size of dataset), we consider the following generalized RP algorithm that combines kernel methods to leverage node feature information in the propagation process
\begin{equation} \label{eqn_rpfeat}
\left[\mathbf R_{t+1}, \mathbf R'_{t+1}\right] =  -\eta \mathbf A^{K} \mathbf K(\bar{\mathbf X}, \bar{\mathbf X})\mathbf A^{K}  [\mathbf R_t, \mathbf 0]  + \left[\mathbf R_t, \mathbf R'_{t}\right], \text{~~where~} \mathbf R_0 = \mathbf Y, \mathbf R'_0 = \mathbf 0,
\end{equation}
where $\mathbf K(\bar{\mathbf X}, \bar{\mathbf X}) \in \mathbb R^{n\times n}$ could be specified as arbitrary kernel functions (such as Sigmoid kernel, Gaussian kernel, etc.) that is applied to compute pairwise similarities. This variant of RP could also be treated as propagation on kernel's RKHS, i.e.
\begin{equation}
\mathbf A^{K} \mathbf K(\bar{\mathbf X}, \bar{\mathbf X})\mathbf A^{K} = (\mathbf A^{K} \mathbf K(\bar{\mathbf X}, \cdot) ) (\mathbf A^{K} \mathbf K(\bar{\mathbf X}, \cdot) )^\top,
\end{equation}
which is impossible to directly implement in practice but can be achieved by our proposed RP.

With this implementation, RP can still run efficiently by treating the computation of the kernel matrix as a part of data prepossessing. Specifically, we will test Gaussian kernel which is defined as
\begin{equation}
\mathbf K(\boldsymbol x_i, \boldsymbol x_j) = \exp \left(-\frac{\left\|\boldsymbol{x_i}-\boldsymbol{x_j}\right\|^2}{2 \sigma^2}\right),
\end{equation}
while in practice one can treat the kernel function as a hyperparameter to tune for even better performance. The pseudo code for this version of RP are shown in the following algorithm.

\begin{algorithm}[H]
    \caption{Generalized residual propagation with kernel functions.}
    \KwIn{Raw graph adjacency matrix ${\mathbf A}$, ground-truth labels for training samples $\mathbf Y$, input features $\mathbf X$ and $\mathbf X'$, step size $\eta$, power $K$.}
     Compute the normalized graph adjacency matrix by $\mathbf A \gets {\mathbf D}^{-\frac{1}{2}} ({\mathbf A + \mathbf I}) {\mathbf D}^{-\frac{1}{2}}$ \\
     Compute the kernel matrix based on input features $\mathbf K([\mathbf X, \mathbf X'], [\mathbf X, \mathbf X'])$\\
     Initialize $\mathbf R_0 \gets \mathbf Y \in \mathbb R^{n_l\times c}$\\
     Initialize $\mathbf R'_0 \gets \mathbf 0 \in \mathbb R^{(n-n_l)\times c}$\\
    \While{Validation performance increases}{
        Label/residual propagation on the graph:\\
        $\tilde{\mathbf R}_{t-1} \gets [\mathbf R_{t-1}, \mathbf 0] \in \mathbb R^{n\times c}$\\
        \For{$i \gets 1$ to $K$}{
            $\tilde{\mathbf R}_{t-1} \gets \mathbf A ~\tilde{\mathbf R}_{t-1}$\\
        }
        $\tilde{\mathbf R}_{t-1} \gets \mathbf K([\mathbf X, \mathbf X'], [\mathbf X, \mathbf X']) ~\tilde{\mathbf R}_{t-1}$\\
        \For{$i \gets 1$ to $K$}{
            $\tilde{\mathbf R}_{t-1} \gets \mathbf A ~\tilde{\mathbf R}_{t-1}$\\
        }
        Update residuals $[\mathbf R_t, \mathbf R'_t] \gets [\mathbf R_{t-1}, \mathbf R'_{t-1}] - \eta \tilde{\mathbf R}_{t-1}$
    }
    Output prediction for testing samples ${\mathbf F'} =  - \mathbf R'$
\end{algorithm}

\subsection{Homophilic Datasets} \label{app_addexp2}
To evaluate the generalized RP algorithm in (\ref{eqn_rpfeat}), we experiment on $7$ more (smaller) datasets: \texttt{Cora}, \texttt{Citeseer}, \texttt{Pubmed}, \texttt{Computer}, \texttt{Photo}, \texttt{CS}, \texttt{Physics}. For \texttt{Cora}, \texttt{Citeseer}, \texttt{Pubmed}, we follow the public split, while for other datasets, we randomly split them into training/validation/testing sets based on ratio $8/1/1$. Statistics of these datasets are reported in Table.~\ref{tbl:dataset}. For RP, we tune the hyperparameters $K$ and $\sigma$. For baselines~\citep{wu2019simplifying,kipf2016semi,xu2018representation,klicpera2018predict}, we tune the hyperparameters provided in their original paper and report mean and standard deviation of testing accuracy with $20$ different runs. 

The results are reported in Table.~\ref{tbl_smallbench}. We found the proposed RP almost always achieves the best or second best performance, and outperforms GCN in $6$ out of $7$ datasets even using no learnable parameters. In terms of the average performance, RP achieved the highest ranking out of all popular GNN models considered. Better performance can potentially be achieved by considering more advanced kernel functions or other specialized propagation matrices.

\subsection{Heterophilic Datasets} \label{app_addexp3}
For a comprehensive evaluation, we further consider $5$ heterophilic benchmarks \texttt{roman-empire}, \texttt{amazon-ratings}, \texttt{minesweeper}, \texttt{tolokers}, \texttt{questions} from a recent paper \citet{platonov2023critical}, which has addressed some drawbacks of existing datasets used for evaluating models designed specifically for heterophily. We use 10 existing standard train/validation/test splits provided in their paper, and statistics of these datasets are also reported in Table.~\ref{tbl:dataset}. Baselines are recently proposed strong GNN models that are carefully designed to tackle the heterophily problem. 

The results are reported in Table.~\ref{tbl_hetebench}. We observed that the proposed RP maintains a strong level of performance when compared to these meticulously designed models, surpassing $6$ out of the $9$ in terms of average performance. It is important to note that RP has been proven to be suboptimal when applied to heterophilic graphs and is orthogonal to various other techniques designed for addressing this challenge. These observations suggest that there is significant untapped potential for further improvement of the algorithm to address this issue.

\begin{table}[t!]
\centering
\caption{Performance of generalized RP on homophilic datasets. We mark the first and second place with gold and silver, and compare its performance with GCN by $\Delta_{GCN}$.}  \label{tbl_smallbench}
\resizebox{\textwidth}{!}{
\begin{tabular}{@{}c|cccccccc@{}}
\toprule
\textbf{Model} &  \texttt{Cora} & \texttt{Citeseer} & \texttt{Pubmed} & \texttt{Computer} &  \texttt{Photo} &  \texttt{CS} &  \texttt{Physics} & Avg.\\
\midrule
MLP & 59.7 $\pm$ 1.0 & 57.1 $\pm$ 0.5 & 68.4 $\pm$ 0.5 & 85.42 $\pm$ 0.51 & 92.91 $\pm$ 0.48 & 95.97 $\pm$ 0.22 & \cellcolor{Goldenrod!40}{96.90 $\pm$ 0.2}7& 79.49\\
SGC &  81.0 $\pm$ 0.5 & \cellcolor{gray!30}{71.9 $\pm$ 0.5} & 78.9 $\pm$ 0.4 & 89.92 $\pm$ 0.37 & 94.35 $\pm$ 0.19 & 94.00 $\pm$ 0.30 & 96.19 $\pm$ 0.13 & 86.61\\
GCN & 81.9 $\pm$ 0.5 & 71.6 $\pm$ 0.4 & 79.3 $\pm$ 0.3 &  \cellcolor{Goldenrod!40}{92.25 $\pm$ 0.61} & 95.16 $\pm$ 0.92 & 94.10 $\pm$ 0.34 & 96.64 $\pm$ 0.36 & 87.28\\
JKNet & 81.3 $\pm$ 0.5 & 69.7 $\pm$ 0.2 & 78.9 $\pm$ 0.6 & 91.25 $\pm$ 0.76 & 94.82 $\pm$ 0.22 & 93.57 $\pm$ 0.49 & 96.31 $\pm$ 0.29 & 86.55\\
APPNP & \cellcolor{gray!30}{82.6 $\pm$ 0.2} & 71.7 $\pm$ 0.5 & \cellcolor{Goldenrod!40}{80.3 $\pm$ 0.1} & 91.81 $\pm$ 0.78 & \cellcolor{Goldenrod!40}{95.84 $\pm$ 0.34} & \cellcolor{gray!30}{94.41 $\pm$ 0.29} & \cellcolor{gray!30}{96.84 $\pm$ 0.26} & \cellcolor{gray!30}{87.64}\\
\midrule
RP (Ours) & \cellcolor{Goldenrod!40}{82.7 $\pm$ 0.0} & \cellcolor{Goldenrod!40}{73.0 $\pm$ 0.0} & \cellcolor{gray!30}{80.1 $\pm$ 0.0} & \cellcolor{gray!30}{92.00 $\pm$ 0.00} & \cellcolor{gray!30}{95.55 $\pm$ 0.00} & \cellcolor{Goldenrod!40}{94.60 $\pm$ 0.00} & 96.75 $\pm$ 0.00 & \cellcolor{Goldenrod!40}{87.81}\\
$\Delta_{\text{GCN}}$ & \red{+ 0.8} & \red{+ 1.4} & \red{+ 1.2} & \blue{- 0.25} & \red{+ 0.39} & \red{+ 0.19} & \red{+ 0.11} & \red{+ 0.53}\\
\bottomrule
\end{tabular}}
\end{table}

\begin{table}[t]
\centering
\caption{Performance of generalized RP on heterophilic datasets. The first three datasets use Accuracy, and the last two datasets use ROC-AUC. We report the ranking of RP among all baselines.}  \label{tbl_hetebench}
\resizebox{\textwidth}{!}{
\setlength{\tabcolsep}{2mm}{
\begin{tabular}{@{}c|cccccc@{}}
\toprule
\textbf{Model} &  {roman-empire} & {amazon-ratings} & {minesweeper} & {tolokers}&  {questions} &   Avg.\\
\midrule
ResNet & 65.88 $\pm$ 0.38 & 45.90 $\pm$ 0.52 & 50.89 $\pm$ 1.39 & 72.95 $\pm$ 1.06 & 70.34 $\pm$ 0.76 & 61.39\\
H2GCN~\citep{zhu2020beyond} & 60.11 $\pm$ 0.52 & 36.47 $\pm$ 0.23 & 89.71 $\pm$ 0.31 & 73.35 $\pm$ 1.01 & 63.59 $\pm$ 1.46 & 64.64\\
CPGNN~\citep{zhu2021graph} & 63.96 $\pm$ 0.62 & 39.79 $\pm$ 0.77 & 52.03 $\pm$ 5.46 & 73.36 $\pm$ 1.01 & 65.96 $\pm$ 1.95 & 59.02 \\
GPR-GNN~\citep{chien2020adaptive} & 64.85 $\pm$ 0.27 & 44.88 $\pm$ 0.34 & 86.24 $\pm$ 0.61 & 72.94 $\pm$ 0.97 & 55.48 $\pm$ 0.91 & 64.88 \\
FSGNN~\citep{maurya2022simplifying} & 79.92 $\pm$ 0.56 & 52.74 $\pm$ 0.83 & 90.08 $\pm$ 0.70 & 82.76 $\pm$ 0.61 & 78.86 $\pm$ 0.92 & 76.87 \\
GloGNN~\citep{li2022finding} & 59.63 $\pm$ 0.69 & 36.89 $\pm$ 0.14 & 51.08 $\pm$ 1.23 & 73.39 $\pm$ 1.17 & 65.74 $\pm$ 1.19 & 57.35\\
FAGCN~\citep{bo2021beyond} & 65.22 $\pm$ 0.56 & 44.12 $\pm$ 0.30 & 88.17 $\pm$ 0.73 & 77.75 $\pm$ 1.05 & 77.24 $\pm$ 1.26 & 70.50\\
GBK-GNN~\citep{du2022gbk} & 74.57 $\pm$ 0.47 & 45.98 $\pm$ 0.71 & 90.85 $\pm$ 0.58 & 81.01 $\pm$ 0.67 & 74.47 $\pm$ 0.86 & 73.58 \\
JacobiConv~\citep{wang2022powerful} & 71.14 $\pm$ 0.42 & 43.55 $\pm$ 0.48 & 89.66 $\pm$ 0.40 & 68.66 $\pm$ 0.65 & 73.88 $\pm$ 1.16 & 69.38\\
\midrule
RP (Ours) & 66.01 $\pm$ 0.56 & 47.95 $\pm$ 0.57 & 80.48 $\pm$ 0.76 & 78.05 $\pm$ 0.90 & 76.39 $\pm$ 1.16 & 69.78\\
\textbf{Ranking} & \textbf{4 / 10} & \textbf{2 / 10} & \textbf{7 / 10} & \textbf{3 / 10} & \textbf{3 / 10} & \textbf{4 / 10} \\
\bottomrule
\end{tabular}}}
\end{table}


\end{document}